\documentclass[10pt,twocolumn,letterpaper]{article}

\usepackage[pagenumbers]{cvpr} 

\usepackage{multirow}
\usepackage{savesym}
\usepackage{bbding}
\usepackage{pifont}
\usepackage{amsmath,amsfonts}
\usepackage{amsthm}
\usepackage{algorithm}
\usepackage{algpseudocode}
\usepackage{xcolor}
\usepackage{colortbl}
\definecolor{TRColor}{RGB}{240, 240, 240}
\usepackage{tabulary}
\usepackage{tabularx}
\usepackage[accsupp]{axessibility}  

\savesymbol{Cross} %
\usepackage{marvosym} 

\newcommand{\X}{\mathcal{X}}
\newcommand{\Y}{\mathcal{Y}}
\newcommand{\XY}{\mathcal{XY}}
\newcommand{\YX}{\mathcal{YX}}

\newtheorem{theorem}{Theorem}[section]
\newtheorem{definition}[theorem]{Definition}

\newtheorem*{unnumberedthm}{Theorem}

\definecolor{cvprblue}{rgb}{0.21,0.49,0.74}
\usepackage[pagebackref,breaklinks,colorlinks,allcolors=cvprblue]{hyperref}

\def\confName{CVPR}
\def\confYear{2026}

\title{From Feature Learning to Spectral Basis Learning: A Unifying and Flexible Framework for Efficient and Robust Shape Matching}

\author{
    Feifan Luo\textsuperscript{\rm 1} \quad
    Hongyang Chen\textsuperscript{\rm 2}\thanks{Corresponding author.} \\
    \textsuperscript{1}\, College of Computer Science and Technology, Zhejiang University, China \quad \\
    \textsuperscript{2}\, Research Center for   Computational Earth and Space Science, Zhejiang Lab, China \\
    {\tt\small luoff@zju.edu.cn}\quad {\tt\small dr.h.chen@ieee.org}
}

\begin{document}
\maketitle

\begin{abstract}
Shape matching is a fundamental task in computer graphics and vision, with deep functional maps becoming a prominent paradigm. However, existing methods primarily focus on learning informative feature representations by constraining pointwise and functional maps, while neglecting the optimization of the spectral basis—a critical component of the functional map pipeline. This oversight often leads to suboptimal matching results. Furthermore, many current approaches rely on conventional, time-consuming functional map solvers, incurring significant computational overhead. To bridge these gaps, we introduce Advanced Functional Maps, a framework that generalizes standard functional maps by replacing fixed basis functions with learnable ones, supported by rigorous theoretical guarantees. Specifically, the spectral basis is optimized through a set of learned inhibition functions. Building on this, we propose the first unsupervised spectral basis learning method for robust non-rigid 3D shape matching, enabling the joint, end-to-end optimization of feature extraction and basis functions. Our approach incorporates a novel heat diffusion module and an unsupervised loss function, alongside a streamlined architecture that bypasses expensive solvers and auxiliary losses. Extensive experiments demonstrate that our method significantly outperforms state-of-the-art feature-learning approaches, particularly in challenging non-isometric and topological noise scenarios, while maintaining high efficiency. Finally, we reveal that optimizing basis functions is equivalent to spectral convolution, where inhibition functions act as filters. This insight enables enhanced representations inspired by spectral graph networks, opening new avenues for future research. Our code is available at \url{https://github.com/LuoFeifan77/Unsupervised-Spectral-Basis-Learning}.
\end{abstract}
    
\section{Introduction}
\label{sec:intro}
Establishing meaningful correspondences between non-rigid shapes is a cornerstone of geometry processing and computer vision, facilitating a myriad of applications such as deformation transfer~\cite{Sumner2004}, shape interpolation~\cite{Eisenberger2021}, and statistical shape analysis~\cite{Bogo2014FAUST}. Recently, the advent of deep learning has catalyzed the development of various data-driven matching frameworks, most notably those centered on the functional map representation~\cite{Ovsjanikov2012}. FMNet~\cite{litany2017deep} pioneered this paradigm by leveraging learned descriptors to compute functional maps in a differentiable manner. Building upon this, subsequent state-of-the-art (SOTA) studies have introduced diverse regularization strategies to enhance feature expressivity, either by constraining pointwise and functional mappings independently~\cite{halimi2019unsupervised,Ayguen2020,Donati2022,li2022learning} or via joint optimization~\cite{attaiki2023understanding,Cao2023,cao2024revisiting,cao2024synchronous,luo2025deep}. These methods typically employ advanced backbones, such as DiffusionNet~\cite{sharp2022diffusionnet}, to extract high-quality spectral features (see~\cref{fig2: DeepFMaps_architecture}).

However, these single-feature learning approaches neglect the optimization of the spectral basis—an essential yet often overlooked component in the (deep) functional map pipeline, consequently leading to suboptimal performance in challenging matching scenarios. Furthermore, existing methods such as \cite{Cao2023, luo2025deep} rely heavily on standard functional map solvers and multiple auxiliary losses, resulting in complex training paradigms and low computational efficiency. 

To address these limitations, we systematically propose Advanced Functional Maps, a framework that extends conventional functional maps \cite{Ovsjanikov2012} by transitioning from fixed to learnable basis functions. These learnable bases consist of both Laplacian eigenfunctions and a set of inhibition functions, where basis expressiveness is enhanced by learning the inhibition functions in a data-driven manner. Building on this framework, we introduce the first end-to-end, unsupervised spectral basis learning method that simultaneously optimizes both shape feature representations and basis functions for non-rigid 3D shape matching. A detailed comparison of methods can be found in Tab.~\ref{tab: others V.S. Ours}. Our method includes: (1) a multi-scale, eigenvalue-agnostic heat diffusion method yields learned basis functions with \textit{minimal} parameters; (2) a multi-resolution unsupervised spectral loss that jointly supervises pointwise maps, functional maps, and basis functions; and (3) the streamlined learning architecture without the computationally intensive functional maps solver or auxiliary functional map losses. Extensive experiments demonstrate that our method not only achieves state-of-the-art performance across various challenging benchmarks but also maintains remarkable computational efficiency. Furthermore, we reveal that optimizing basis functions is equivalent to performing spectral convolution, where inhibition functions act as filters. This theoretical connection enables the design of superior basis representations inspired by spectral graph/manifold convolution techniques, opening new avenues for future research. Our main contributions are summarized as follows:
\begin{itemize}
\item We propose Advanced Functional Maps, a generalized framework where spectral basis functions are enhanced through a set of learnable inhibition functions.
\item Based on this framework, we present the first unsupervised spectral basis learning method for non-rigid 3D shape matching, enabling joint optimization of features and bases.
\item We demonstrate through extensive experiments that our approach significantly outperforms existing feature-learning methods in both accuracy and computational efficiency across challenging benchmarks.
\item We establish a formal equivalence between basis optimization and filter design in spectral GNNs, providing a theoretical foundation for future explorations in basis learning.
\end{itemize}

\begin{table}[h!t]
\centering
\caption{Method comparison. Our method is the first unsupervised spectral basis learning approach tailored specifically for non-rigid 3D shape matching, incorporating a unique set of properties that enhance its performance. Where FL: \underline{f}eature \underline{l}earning. BL: \underline{b}asis \underline{l}earning. FU: \underline{f}ull \underline{u}nsupervised. WoFS: \underline{w}ith\underline{o}ut \underline{f}unctional map \underline{s}olver. SL: \underline{s}ingle unsupervised \underline{l}oss for networks.} 
\scalebox{0.9}{
\begin{tabular}{lcccccccccccccc}
\toprule
            & \multicolumn{2}{c}{FL}  & \multicolumn{2}{c}{BL}  & \multicolumn{2}{c}{FU} & \multicolumn{2}{c}{WoFS} & \multicolumn{2}{c}{SL} \\ \midrule

FMNet~\cite{litany2017deep}  & \multicolumn{2}{c}{\color{green}\ding{51}}   & \multicolumn{2}{c}{\color{red}\ding{55}} & \multicolumn{2}{c}{\color{red}\ding{55}} & \multicolumn{2}{c}{\color{red}\ding{55}} & \multicolumn{2}{c}{-}\\ 
            
GeomFmaps~\cite{donati2020deep}  & \multicolumn{2}{c}{\color{green}\ding{51}}   & \multicolumn{2}{c}{\color{red}\ding{55}} & \multicolumn{2}{c}{\color{red}\ding{55}} & \multicolumn{2}{c}{\color{red}\ding{55}}  & \multicolumn{2}{c}{-}\\ 

DiffMaps~\cite{marin2020correspondence} & \multicolumn{2}{c}{\color{green}\ding{51}} & \multicolumn{2}{c}{\color{green}\ding{51}}  & \multicolumn{2}{c}{\color{red}\ding{55}}  & \multicolumn{2}{c}{\color{green}\ding{51}} & \multicolumn{2}{c}{-} \\ 

DUO-FMNet~\cite{donati2022deep}    & \multicolumn{2}{c}{\color{green}\ding{51}}   & \multicolumn{2}{c}{\color{red}\ding{55}}  & \multicolumn{2}{c}{\color{green}\ding{51}}  & \multicolumn{2}{c}{\color{red}\ding{55}} & \multicolumn{2}{c}{\color{red}\ding{55}} \\ 

AttentiveFMaps~\cite{li2022learning}    & \multicolumn{2}{c}{\color{green}\ding{51}}   & \multicolumn{2}{c}{\color{red}\ding{55}}  & \multicolumn{2}{c}{\color{green}\ding{51}}  & \multicolumn{2}{c}{\color{red}\ding{55}} & \multicolumn{2}{c}{\color{red}\ding{55}}\\ 

RFMNet~\cite{HU2023101189}  & \multicolumn{2}{c}{\color{green}\ding{51}}  & \multicolumn{2}{c}{\color{red}\ding{55}}  & \multicolumn{2}{c}{\color{green}\ding{51}}  & \multicolumn{2}{c}{\color{green}\ding{51}} & \multicolumn{2}{c}{\color{green}\ding{51}}\\ 

ULRSSM~\cite{Cao2023}  & \multicolumn{2}{c}{\color{green}\ding{51}} & \multicolumn{2}{c}{\color{red}\ding{55}}  & \multicolumn{2}{c}{\color{green}\ding{51}} & \multicolumn{2}{c}{\color{red}\ding{55}} & \multicolumn{2}{c}{\color{red}\ding{55}}\\ 

DiffZO~\cite{magnet2024memory}   & \multicolumn{2}{c}{\color{green}\ding{51}}  & \multicolumn{2}{c}{\color{red}\ding{55}}  & \multicolumn{2}{c}{\color{green}\ding{51}} & \multicolumn{2}{c}{\color{green}\ding{51}} & \multicolumn{2}{c}{\color{red}\ding{55}}\\ 

HybridFmaps~\cite{bastian2024hybrid}  & \multicolumn{2}{c}{\color{green}\ding{51}}  & \multicolumn{2}{c}{\color{red}\ding{55}}  & \multicolumn{2}{c}{\color{green}\ding{51}}  & \multicolumn{2}{c}{\color{red}\ding{55}} & \multicolumn{2}{c}{\color{red}\ding{55}}\\ 

DeepFAFM~\cite{luo2025deep}  & \multicolumn{2}{c}{\color{green}\ding{51}}  & \multicolumn{2}{c}{\color{red}\ding{55}}  & \multicolumn{2}{c}{\color{green}\ding{51}}  & \multicolumn{2}{c}{\color{red}\ding{55}} & \multicolumn{2}{c}{\color{red}\ding{55}}\\

DenoiseFmaps~\cite{zhuravlev2025denoising}  & \multicolumn{2}{c}{\color{green}\ding{51}}  & \multicolumn{2}{c}{\color{red}\ding{55}}  & \multicolumn{2}{c}{\color{green}\ding{51}}  & \multicolumn{2}{c}{-} & \multicolumn{2}{c}{\color{red}\ding{55}}\\

Ours & \multicolumn{2}{c}{\color{green}\ding{51}}   & \multicolumn{2}{c}{\color{green}\ding{51}}  & \multicolumn{2}{c}{\color{green}\ding{51}}  & \multicolumn{2}{c}{\color{green}\ding{51}} & \multicolumn{2}{c}{\color{green}\ding{51}}\\
\bottomrule
\end{tabular}}
\label{tab: others V.S. Ours}
\end{table}

\section{Related Works}
\label{sec:rw}
Non-rigid 3D shape matching is a long-standing problem that has been extensively studied. For a comprehensive overview of the field, readers are referred to surveys such as \cite{sahilliouglu2020recent,deng2022survey}. Here, we focus on approaches most similar to ours.

\subsection{Functional map methods}
The functional map framework~\cite{Ovsjanikov2012} serves as a foundational methodology for non-rigid shape matching, which has inspired extensive subsequent research~\cite{ren2018continuous, Ren2019, melzi2019matching, Huang2020, Hu2021, ren2021discrete, Gao_2021_CVPR, fan2022coherent, 2022SmoothNonRigidShapeMatchingviaEffectiveDirichletEnergyOptimization, Donati2022, 2023ElasticBasis}. These methods incorporate functional or map constraints to enhance matching accuracy and robustness. However, axiomatic functional map methods rely heavily on the quality of extrinsic~\cite{salti2014shot} and intrinsic~\cite{sun2009concise,Aubry2011The,lhllcgf2021, Liu2023} handcrafted features. Their performance deteriorates under large-scale deformations, often producing unsatisfactory results. Additionally, the traditional functional maps framework is defined on fixed basis functions, we propose a novel advanced functional maps framework, which is defined on adjustable basis functions, providing a flexible and general matching pipeline. 

\subsection{Deep functional map methods}
Deep functional map methods depart from traditional axiomatic approaches by replacing handcrafted features with data-driven feature learning. The pioneering deep functional map framework, FMNet~\cite{litany2017deep}, introduced an end-to-end learnable pipeline that enhances matching accuracy by optimizing SHOT descriptors~\cite{salti2014shot} using residual MLP layers. Building on this, subsequent works~\cite{halimi2019unsupervised, Ayguen2020} integrated structured regularizers~\cite{roufosse2019unsupervised, sharma2020weakly}—such as bijectivity and orthogonality~\cite{Ren2019}—to further enhance map representations. Recent architectural advancements have further expanded the field: two-branch designs~\cite{Cao2023, Sun_2023_ICCV, attaiki2023understanding, cao2024spectral, le2024integrating,cao2024revisiting, bastian2024hybrid,luo2025deep,cao2024synchronous} promote the \textit{properness} of functional maps~\cite{ren2021discrete}, while streamlined single-branch approaches~\cite{HU2023101189, magnet2024memory} focus on computational efficiency. More recently, the field has witnessed a surge of diverse methodologies, including Finsler metrics~\cite{weber2024finsler}, Wormhole loss~\cite{bracha2024wormhole}, neural adjoint maps~\cite{vigano2025nam}, and diffusion model-based matching paradigms~\cite{pierson2025diffumatch, zhuravlev2025denoising}, alongside innovative architectures such as SpiderMatch~\cite{roetzer2024spidermatch}, EchoMatch~\cite{xie2025echomatch}, and other studies~\cite{bensaid2023partial,bracha2024unsupervised,yona2025neural,rimon2025fridu}, collectively pushing the boundaries of non-rigid shape correspondence.
Despite these advances, existing research focuses almost exclusively on feature learning while leaving spectral basis optimization largely unexplored, leading to suboptimal performance in challenging scenarios. Moreover, heavy reliance on iterative linear solvers and complex auxiliary losses compromises computational efficiency. In contrast, we introduce the first unsupervised spectral basis learning method that jointly optimizes features and basis functions. By enabling task-specific basis adaptation, our approach significantly outperforms state-of-the-art techniques~\cite{Cao2023, magnet2024memory, bastian2024hybrid} while minimizing overhead and dependency on elaborate regularizers.

\subsection{Spectral Basis functions}
Recent advances in the field have increasingly highlighted the pivotal role of the spectral basis in enhancing both the expressive power and accuracy of matching results. Various studies have sought to improve mapping quality by constructing specialized variations of the Laplacian eigenbasis, including the coupled quasi-harmonic~\cite{Kokkinos2013Intrinsic}, local harmonic~\cite{melzi2018localized}, Hamiltonian~\cite{choukroun2018hamiltonian}, and landmark-adapted bases~\cite{panine2022non}. Others have introduced orientation-aware complex bases~\cite{Donati2022} or extrinsic crease-aware elastic bases~\cite{2023ElasticBasis, bastian2024hybrid}. Despite their efficacy, these methods are fundamentally constrained by their reliance on axiomatic and fixed basis functions, which precludes the use of data-driven strategies for further basis optimization. In the context of 3D point cloud correspondence, DiffFMaps~\cite{marin2020correspondence} proposed a supervised approach by training linear spectral embeddings and shape features independently. However, this framework relies heavily on ground-truth labels and entails a decoupled, computationally intensive training procedure. Furthermore, because its spectral embeddings are derived from disconnected point clouds, it is not directly applicable to mesh-based matching. To bridge these gaps, we introduce a novel unsupervised spectral basis learning method for mesh data. By eliminating the need for labeled data and co-optimizing feature extraction and basis functions, our approach offers a more integrated and effective paradigm for basis optimization.

\begin{figure}[h!t]
	\centering
	\includegraphics[width=1.0\linewidth]{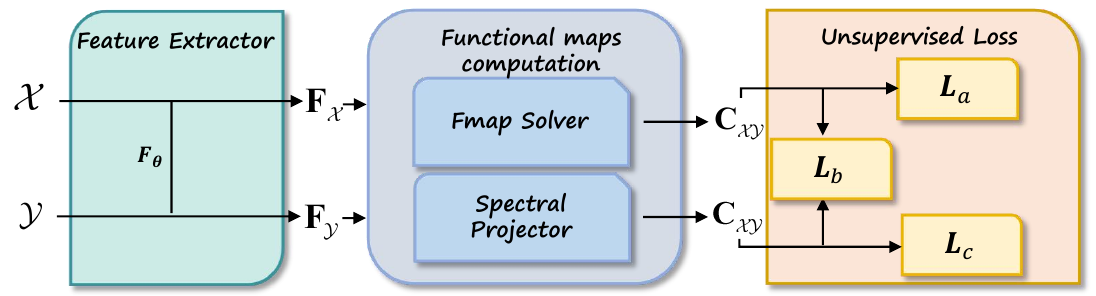}
	\caption{Existing deep functional map architectures. Dual-branch: Integrates both a functional map solver and a spectral projector to compute functional maps, supervised by unsupervised losses (e.g., orthogonality, bijectivity) and coupled with a consistency loss between the two functional representations. 
    Single-branch: Employs either a functional map solver or a spectral projector alone, regulated by corresponding unsupervised constraints.}
	\label{fig2: DeepFMaps_architecture}
\end{figure}
\section{Background and notation}\label{sec:bg}
We begin with a concise overview of the deep functional map pipeline, referring interested readers to~\cite{Ovsjanikov2012, Melzi2019, Cao2023} for a more comprehensive treatment. For clarity, the primary notations and symbols used throughout this paper are summarized in~\cref{tab: symbol desc}. Generally, deep functional map frameworks comprise five principal stages, as illustrated in~\cref{fig2: DeepFMaps_architecture}:

\begin{table}[h!t]
\centering
\caption{Summary of the symbol description.}
\scalebox{0.80}{
\begin{tabular}{lcl}
\toprule
{Symbol}     &         & Description           \\ \midrule
$\X, \Y$  &   & 3D shapes (trangle mesh) with $|V_\mathcal{X}|$, $|V_\Y|$ vertices      \\ 
$\Lambda_{\X}$  &  & $\mathbb{R}^{k \times k}$ eigenvalue matrix of shape $\X$    \\
$\Phi_{\X,k}$ &  & $\mathbb{R}^{|V_\mathcal{X}| \times k}$ eigenfunctions  matrix   \\
$\Phi^{\dagger}_{\X,k}$ &  & $\mathbb{R}^{k \times |V_\mathcal{X}|}$ Moore-Penrose inverse of $\Phi_{\X}$     \\
$M_{\X}$ &  & $\mathbb{R}^{|V_\mathcal{X}| \times k}$ mass matrix.   \\
$F_{\X}$ &  & $\mathbb{R}^{|V_\mathcal{X}| \times d}$ vertex-wise features of shape ${\X}$     \\
$\Pi_{\YX}$ &  & $\mathbb{R}^{|V_\mathcal{Y}| \times |V_\mathcal{X}| }$ pointwise map between shapes $\Y$ and $\X$    \\
$\Psi_{\X,k}$ &  & $\mathbb{R}^{|V_\mathcal{X}| \times k}$ adaptable eigenfunctions matrix   \\
$C_{\XY}$/$C^{\mathrm{A}}_{\XY}$ &  & $\mathbb{R}^{k \times k}$ standard/advanced functional maps   \\
$||\cdot||_{\mathrm{F}}$ &  & Frobenius norm \\
\bottomrule
\end{tabular}}
\label{tab: symbol desc}
\end{table}

\textbf{(i) Precomputation.} Compute the first $k$ orthonormal Laplacian eigensystem $\{\Phi_{\X,k}, \Lambda_{\X,k}\}$ and $\{\Phi_{\Y,k}, \Lambda_{\Y,k}\}$ in matrix notation for shape $\X$ and $\Y$ via generalized Laplacian eigendecomposition~\cite{pinkall1993computing}, respectively.

\textbf{(ii) Feature extraction.} The feature functions on each shape are extracted by a feature extractor network $\mathcal{F}_\Theta$ (DiffusionNet~\cite{sharp2022diffusionnet} for most), with shared parameters $\Theta$, i.e., $F_{\mathcal{X}} = \mathcal{F}_\Theta(\mathcal{X})$ and $F_{\mathcal{Y}} = \mathcal{F}_\Theta(\mathcal{Y})$, respectively.

\textbf{(iii) Functional maps computation.} The standard functional map $C_{\mathcal{XY}}$ is encoded as a $k\times k$ transfer matrix~\cite{Ovsjanikov2012} between \textit{Fourier} coefficients of $F_{\mathcal{X}}$ and $F_{\mathcal{Y}}$, which can be computed by solving the optimisation problem, i.e.,
\begin{small}
\begin{equation}\label{equ: desc and reg}
   {C}_{\mathcal{XY}} = \mathop{\arg\min}\limits_{{C}_{\mathcal{XY}}} \left\|   {C}_{\mathcal{XY}} \Phi_{\mathcal{X},k}^{\dagger}  F_{\mathcal{Y}} - \Phi_{\mathcal{Y},k}^{\dagger}  F_{\mathcal{X}} \right\|^{2}_{\mathrm{F}} + \mu E_{reg}({C}_{\mathcal{XY}}), 
\end{equation}
\end{small}
where  
$ \Phi^{\dagger}_{\mathcal{X}}=\Phi^{\top}_{\mathcal{X}}  M_{\mathcal{X}} $,  
$E_{reg} = \left\| C_{\mathcal{XY}}\Lambda_{\mathcal{X}} - \Lambda_{\mathcal{Y}} C_{\mathcal{XY}} \right\|^{2}_{\mathrm{F}}$ denotes the Laplacian commutativity constraint term, and $\mu$ is a hyperparameter. 

Unlike the aforementioned approach of solving a linear system via least squares, an alternative explicit computational method employs spectral basis projection, expressed as:
\begin{equation}\label{equ: compute C by Pi}
{C}_\mathcal{XY}=\Phi_{\mathcal{Y},k}^\dagger\Pi_\mathcal{YX}\Phi_{\mathcal{X},k},
\end{equation}
where the pointwise map $ \Pi_\mathcal{YX} \in \mathbb{R}^{|V_{\mathcal{Y}}| \times |V_{\mathcal{X}}|} $, satisfying $\Pi_\mathcal{YX} F_\mathcal{X} = F_\mathcal{Y}$. Under specific regularity conditions, the spectral and spatial formulations, i.e., Eq~\eqref{equ: desc and reg} and Eq~\eqref{equ: compute C by Pi}, are equivalent~\cite{cao2024revisiting}.

\textbf{(iv) Loss functions.} During the training stage, the feature functions are penalized by unsupervised loss terms enforcing properties such as bijectivity, orthogonality on the functional maps~\cite{Ren2019}, i.e., 
\begin{equation}\label{equ: fmap loss}
    L_{bi}  =\left\| {C}_\mathcal{XY} {C}_\mathcal{YX} -\mathrm{I}\right\|^{2}_\mathrm{F}, L_{or}  = \left\|  {C}_\mathcal{XY} {C}_\mathcal{XY}^{\top} -{I}\right\|^{2}_\mathrm{F},   
\end{equation}
or the coupling~\cite{Cao2023} between pointwise and functional maps, i.e., 
\begin{equation}\label{equ: couple loss}
    L_{co}  =\left\| {C}_\mathcal{XY} - \Phi_\mathcal{Y}^\dagger\Pi_\mathcal{YX}\Phi_\mathcal{X}\right\|^{2}_\mathrm{F}.  
\end{equation}
Additionally, the aforementioned loss terms can be extended to bidirectional constraints by incorporating reverse mapping terms.

\textbf{(v) Map computation.} During inference, 
a common strategy originates from ZoomOut~\cite{Melzi2019}, which converts the estimated functional map $ {C}_{\mathcal{XY}}$ to a pointwise map by minimize the following  energy:
\begin{equation}\label{equ:area preservation variation fixed}
E({\Pi}_{\mathcal{YX}},{C}_\mathcal{XY}) = \left\|\Phi_{\mathcal{Y},k} - {\Pi}_{\mathcal{YX}}\Phi_{\mathcal{X},k} {C}^{\top}_\mathcal{XY}\right\|^{2}_{\mathrm{F}}, 
\end{equation}
where ${C}_\mathcal{XY}=\Phi_{\mathcal{Y},k}^\dagger\Pi_\mathcal{YX}\Phi_{\mathcal{X},k}$. Finally,~\cref{equ:area preservation variation fixed} is able to be solved via nearest neighbor search~\cite{Ovsjanikov2012,Melzi2019} between the aligned spectral embeddings $\Phi_{\mathcal{Y},k}$ and $\Phi_{\mathcal{X},k}{C}_{\mathcal{XY}}^\top$, i.e., 
\begin{equation}\label{equ:nnsearch Pi}
    \Pi_{\mathcal{YX}} = \text{NS}(\Phi_{\mathcal{Y},k},\Phi_{\mathcal{X},k}{C}_{\mathcal{XY}}^\top).
\end{equation} 

As established, standard deep functional map frameworks primarily optimize shape features, i.e., $Loss(\mathcal{F}_{\Theta}(\mathcal{X}), \mathcal{F}_{\Theta}(\mathcal{Y}))$, as shown in~\cref{fig2: DeepFMaps_architecture}. Despite the fundamental role of the spectral basis in these representations, existing methods—while leveraging data-driven feature optimization, consistently treat the basis functions as fixed, non-optimized components. To bridge this gap, we first introduce Advanced Functional Maps (\cref{sec:AdvancedFM}), a generalized framework built upon learnable spectral bases. Building on this foundation, we propose a novel, lightweight unsupervised method for spectral basis learning in non-rigid 3D shape matching (\cref{sec:BasisLearning}). This approach enables the joint, end-to-end optimization of both features and basis functions, i.e., $Loss(\mathcal{F}_{\Theta}(\mathcal{X}), \mathcal{F}_{\Theta}(\mathcal{Y}), \mathcal{B}_{\mathcal{G}}(\Phi_\mathcal{X}), \mathcal{B}_{\mathcal{G}}(\Phi_\mathcal{Y}))$, as illustrated in~\cref{fig: our_network}. Our streamlined pipeline not only outperforms feature-centric approaches in accuracy but also achieves superior computational efficiency.
\section{Advanced Functional Maps}\label{sec:AdvancedFM}
In this section, we systematically formalize the advanced functional maps framework. Our approach encompasses three core components: Learnable basis representation (\cref{sec:learnable_basis_representation}), which introduces learnable spectral functions; Advanced functional map formulation (\cref{sec:AdvancedFM_method}), which defines the spectral mapping over these adaptive bases; Pointwise map recovery (\cref{sec:map_recovery_from_orth}), detailing the computation of dense correspondences from the orthonormality promoting term.

\subsection{Learnable Spectral Basis Representation}\label{sec:learnable_basis_representation}
\begin{definition}
Given the first $k$ (Laplacian) basis functions $ \Phi_{k}  = [\phi_{1}, \phi_{2} ,\cdots, \phi_{k}]$ defined on a shape, and a set of inhibition functions ${G} = diag\{g_1, g_2, \dots, g_k\} $, where $g_i: \mathbb{R}\to (0,1]$, $ 1 \leq  i \leq k$, then, we define the learnable spectral basis as 
\begin{equation}
    \Psi_k := \Phi_k G,
\end{equation}
where $\Psi= \left[\psi_{1}, \psi_{2}, \dots, \psi_{k} \right] $,  ${\psi}_{i} = {g_i} {\phi}_{i}$.
\end{definition}

We develop a novel learnable basis function representation by introducing inhibition functions $G$ that \textit{linearly} act upon the Laplacian basis $\Phi_k$. These inhibition functions serve to selectively suppress the spectral basis to varying degrees, which can be intuitively regarded as an attention or filtering mechanism. Moreover, this formulation exhibits a number of advantageous properties.
\begin{itemize}
    \item \textit{Invertibility}. 
    The Moore–Penrose pseudoinverse $\Psi_k^{\dagger} = G^{-1}\Phi_k^{\top} {M} $, since $  \Psi_k^{\dagger} \Psi_k = G^{-1}\Phi_k^{\top}{M}\Phi_k G = I $.
    \item \textit{Orthogonality preservation}. The Laplacian basis $\Phi_k$ is orthogonal, then $\Psi_k$ is orthogonal. Since  $<\psi_{i}, \psi_{j}>_M = <g_i\phi_{i}, g_j\phi_{j}>_M = 0$, if $i\neq j$.
    \item \textit{Learnability}. The inhibition functions ${G}$ can be regarded as learnable parameters, which can be learned in a data-driven manner.
    \item \textit{Structure information preservation}. The incorporation of the inhibition function ${G}$ does not alter the prior structural information encoded by the original basis $\Phi_k$. For instance, just as the low-frequency information of the shape is encoded in $\phi_{1}$, it is also preserved in $\psi_{1}$.
\end{itemize}

\subsection{Advanced Functional Map}\label{sec:AdvancedFM_method}
\begin{theorem}
\textbf{Advanced functional maps.} Consider the first $k$ adaptable basis functions $\Psi_{\mathcal{X},k} = \Phi_{\mathcal{X},k} G_{\mathcal{X}} \in \mathbb{R}^{|V_\mathcal{X}|\times k}$ and $\Psi_{\mathcal{Y},k} = \Phi_{\mathcal{Y},k} G_{\mathcal{Y}}\in \mathbb{R}^{|V_\mathcal{Y}|\times k}$, defined on shapes $\mathcal{X}$
 and $\mathcal{Y}$, respectively. The optimization problem in~\cref{equ: desc and reg} is reformulated as: 
 \begin{small}
\begin{equation}\label{equ:desc and reg adj}
   {C}^{\mathrm{A}}_{\mathcal{XY}} = \mathop{\arg\min}\limits_{ {C}^{\mathrm{A}}_{\mathcal{XY}}\!} \left\|  {C}^{\mathrm{A}}_{\mathcal{XY}} \Psi_{\mathcal{Y},k}^{\dagger}   F_{\mathcal{Y}} - \Psi_{\mathcal{X},k}^{\dagger}   F_{\mathcal{X}} \right\|^{2}_{\mathrm{F}} 
    + \lambda {E}_{reg}({C}^{\mathrm{A}}_{\mathcal{XY}}), 
\end{equation} 
\end{small}
where $E_{reg} = \left\| C^{\mathrm{A}}_{\mathcal{XY}}\Lambda_{\mathcal{X}} - \Lambda_{\mathcal{Y}} C^{\mathrm{A}}_{\mathcal{XY}} \right\|^{2}_{\mathrm{F}}$.
\begin{proof}
    The first component ensures that ${C}^{\mathrm{A}}_{\mathcal{XY}}$ acts as an optimal linear mapping between the Fourier coefficients of descriptors $F_{\mathcal{X}}$ and $F_{\mathcal{Y}}$ within the learned basis space. The second component, $E_{reg}$, penalizes the non-commutativity between the functional map and the Laplacian operator as defined on the adaptable basis functions. A detailed proof can be found in the supplementary material (\cref{sec:desc_pre_lam_commu_in_supp}).
\end{proof}
\end{theorem}

Specifically, we refer to ${C}^{\mathrm{A}}_{\mathcal{XY}}$ as an \textit{advanced functional map}, as it is defined over learnable basis functions, contrasting with the fixed, handcrafted bases employed in conventional functional map frameworks~\cite{Ovsjanikov2012}.

Alternatively, we compute the rank-$k$ truncated approximation of the pointwise map $\Pi_\mathcal{YX}$ using the adaptable basis functions and derive the corresponding advanced functional map, namely, 
\begin{equation}\label{equ: compute C by Pi adj}
{C}^{\mathrm{A}}_\mathcal{XY}=\Psi_{\mathcal{Y},k}^\dagger\Pi_\mathcal{YX}\Psi_{\mathcal{X},k}.
\end{equation}
The two computational methods are equivalent under some conditions; please refer to the supplementary materials for further details (\cref{sec:equ_sam_in_supp}).
 
\subsection{Map Recovery}\label{sec:map_recovery_from_orth}
Recovering pointwise maps from functional maps is a canonical step in the standard functional maps pipeline. 
Inspired by ZoomOut~\cite{Melzi2019}, we first enforce that $C^{\mathrm{A}}_\mathcal{XY}$ preserves orthogonality by minimize the following energy, i.e., 
\begin{equation}\label{equ:orth preservation adj}
 \left\|  C^{\mathrm{A}}_\mathcal{XY} (C^{\mathrm{A}}_\mathcal{XY})^{\top} -{I}\right\|^{2}_\mathrm{F}.
\end{equation}


\begin{theorem}
Substituting ${C}^{\mathrm{A}}_\mathcal{XY}=\Psi_{\mathcal{Y},k}^\dagger\Pi_\mathcal{YX}\Psi_{\mathcal{X},k}$  into~\cref{equ:orth preservation adj} derives an alternative orthogonality constraint that simultaneously penalizes components of $\Pi_\mathcal{YX}$ outside the subspace spanned by $\Psi_\mathcal{Y}$, namely:

\begin{equation}\label{equ:area preservation variation adj}
E({\Pi}_{\mathcal{YX}},C^{\mathrm{A}}_\mathcal{XY},G) = \left\|\Psi_{\mathcal{Y},k} - {\Pi}_{\mathcal{YX}}\Psi_{\mathcal{X},k} (C^{\mathrm{A}}_\mathcal{XY})^{\top}\right\|^{2}_{\mathrm{F}}.
\end{equation}
\end{theorem}

\begin{proof}
 For a detailed proof, please refer to the supplementary material (\cref{sec:nn_FMsolver}).
\end{proof}

We assert that~\cref{equ:area preservation variation fixed} and~\cref{equ:area preservation variation adj} differ in their implications: ${C}^{\top}_\mathcal{XY}$ in~\cref{equ:area preservation variation fixed} denotes the \textit{adjoint operator}~\cite{Gautam2021} of ${C}_\mathcal{XY}$, therefore~\cref{equ:area preservation variation fixed} enforces a regularizer that strictly preserves local area for pointwise maps~\cite{ROABC2013}. Whereas~\cref{equ:area preservation variation adj} is designed to promote orthogonal structural properties of ${C}^{\mathrm{A}}_\mathcal{XY}$ consistent with ${C}_\mathcal{XY}$. Additionally, given $G$ and $C^{\mathrm{A}}_\mathcal{XY}$, the pointwise map in~\cref{equ:area preservation variation adj} is solved via nearest neighbor search algorithm, i.e., 
\begin{equation}\label{equ:nnsearch Pi adj}
    \Pi_{\mathcal{YX}} = \text{NS}(\Psi_{\mathcal{Y},k},\Psi_{\mathcal{X},k}({C^{\mathrm{A}}_\mathcal{XY}})^\top).
\end{equation}

Finally, we propose a novel spectral upsampling alternating refinement algorithm, dubbed \textit{G-ZoomOut}, please see~\cref{alg: G-ZoomOut refinement} in the supplementary material.

\begin{figure*}[h!t]
	\centering
	\includegraphics[width=0.9\linewidth]{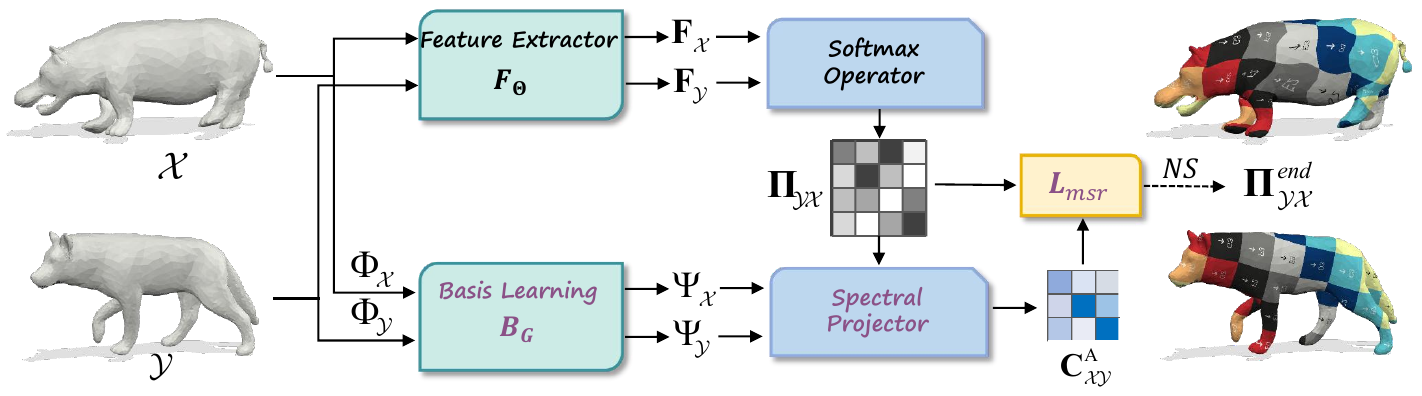}
	\caption{An overview of our method. (1) Feature Extraction: Learned features ${F}_{\mathcal{X}}$ and ${F}_{\mathcal{Y}}$ are extracted from shapes $\mathcal{X}$ and $\mathcal{Y}$, respectively. 
    (2) Basis Learning: The optimized basis functions,  $\Psi_{\mathcal{X}}$ and $\Psi_{\mathcal{Y}}$, are generated by the proposed heat diffusion network.
    (3) Map Estimation: The differentiable pointwise map ${\Pi}_\mathcal{YX}$ is computed using the softmax operator~\cref{eq: compute soft map}, and the advanced functional map ${{C}}^{\mathrm{A}}_\mathcal{XY}$ is calculated via spectral basis projection~\cref{equ: compute C by Pi adj}. 
    (4) An unsupervised loss~\cref{equ:multi_spectral_loss} is constructed to supervise basis functions, pointwise maps, and functional maps.
    (5) The pointwise map ${\Pi}^{end}_\mathcal{YX}$ is recovered during inference using~\cref{equ: nnsearch1 in test} and Eq.\eqref{equ: nnsearch2 in test}.}\label{fig: our_network}
\end{figure*}

\section{Unsupervised Spectral Basis Learning for Efficient and Robust Shape Matching}\label{sec:BasisLearning}

Leveraging the advanced functional map framework (\cref{sec:AdvancedFM}), this section introduces the first unsupervised spectral basis learning approach for non-rigid 3D shape matching. As illustrated in~\cref{fig: our_network}, the proposed architecture comprises several integrated modules: a feature extraction backbone (\cref{sec:feat_extra}), a learnable spectral basis layer (\cref{sec:basis_learning_layer}), and an advanced functional map estimation layer (\cref{sec:axiomatic_functional_map_layer}). The entire pipeline is optimized via a multi-resolution unsupervised loss (\cref{sec:MSLoss}), followed by a specialized inference strategy (\cref{sec:nn_inference}) for computing pointwise correspondences during testing.

\subsection{Feature Extraction}\label{sec:feat_extra}
Extracting local shape geometrical information constitutes a fundamental step in shape matching. Following common practice~\cite{Cao2023,magnet2024memory,bastian2024hybrid,luo2025deep,zhuravlev2025denoising}, we employ the widely recognized SOTA local feature extractor, DiffusionNet~\cite{sharp2022diffusionnet}, to generate the features ${F}_\mathcal{X}$ and ${F}_\mathcal{Y}$ for shape $\mathcal{X}$ and $\mathcal{Y}$, respectively.

\subsection{Basis Learning}\label{sec:basis_learning_layer}
The core of our approach is learning a novel set of basis functions $\Psi$, which consists of two components: the original basis functions $\Phi$, encoding prior geometric information, and the inhibition functions $G$, which enhance the basis expressiveness through a data-driven mechanism. On the other hand, established \textit{spectral graph convolution theory}~\cite{chung1997spectral} and heat diffusion~\cite{chamberlain2021grand} can be seamlessly integrated with our framework, allowing the direct application of spectral graph/manifold convolutional neural networks (CNNs) to optimize the basis functions.


\textit{Spectral convolution.} Spectral convolution is a fundamental operation in the field of graph signal processing~\cite{shuman2013emergingGPS}. The spectral filtering of an input signal $x_{in}$ using a kernel $f$ is defined by the convolution theorem~\cite{arfken2011mathematical} as:
\begin{equation}\label{equ:Spectral Convolution continue}
x_{out}= (\Phi * f) = \Phi f(\Lambda)\Phi^{\dagger}x_{in}, 
\end{equation}
where the filter $f$ is typically parameterized by polynomials, i.e., $f(\lambda) = \sum_{j=0}^J \theta_j p_j(\lambda)$. In the context of undirected graphs, the spectral basis $\Phi$ is orthonormal, implying $\Phi^{\dagger} = \Phi^{\top}$. We employ the pseudoinverse notation $\dagger$ here to maintain mathematical generality across various domains.

\textit{Heat diffusion}. The classical heat diffusion process on a manifold can be formulated in the spectral domain~\cite{sharp2022diffusionnet,behmanesh2023tide} as:
\begin{equation}\label{equ:heat0}
h_t(x_{in}) = \Phi e^{-t\Lambda}\Phi^{\dagger}x_{in}.
\end{equation}

We emphasize a fundamental mathematical unification between heat diffusion and spectral graph convolutions. Although they stem from different signal processing origins, both operations share an identical mathematical formalism~\cite{shuman2016vertex,chamberlain2021grand,behmanesh2023tide} rooted in spectral filtering. Specifically, heat diffusion is characterized by the exponential filter $f(\lambda) = e^{-t\lambda}$, whereas spectral convolutions typically employ polynomial filters. Consequently, heat diffusion can be viewed as a specialized, constrained instance of spectral convolution. Further theoretical discussions and technical details are provided in the Supplementary Material (\cref{sec:gnns_heat}).


\begin{theorem}
By treating the basis functions $\Phi$ as signal embeddings, the spectral convolution operator can be reformulated as:
\begin{equation}\label{equ:Spectral Convolution discrete}
(\Phi * f) = \Phi f(\Lambda)\Phi^{\dagger}\Phi = \Phi f(\Lambda).
\end{equation}
Furthermore, by defining the inhibition function as $G = f(\Lambda)$, we have:
\begin{equation}
\Psi = \Phi f(\Lambda) = (\Phi * f).
\end{equation}
\end{theorem}

This result demonstrates that the proposed basis function learning is formally equivalent to performing spectral convolution on the original basis functions, where the inhibition function $G$ serves as the spectral filter. Consequently, classic spectral graph/manifold CNNs (e.g., ChebyNet~\cite{defferrard2016ChebyNet}, MoNet~\cite{monti2017MoNet}, and ACSCNNs~\cite{Li2020}), as well as heat diffusion frameworks (e.g., GRAND~\cite{chamberlain2021grand}, DiffusionNet~\cite{sharp2022diffusionnet}, Synchronous Diffusion~\cite{cao2024synchronous}), can be interpreted as mechanisms that directly optimize these basis functions.

Building upon these insights, we propose a tailored basis learning network designed to provide more effective, task-specific optimization than general-purpose spectral convolutional architectures.

\paragraph{Multi-scale Eigenvalue-Agnostic Heat Diffusion.} 
The basis optimization based on classical heat diffusion~\cite{sharp2022diffusionnet,cao2024synchronous}, typically truncated to the first $k$ eigenfunctions, can be formalized as:
\begin{equation}\label{equ:heat1}
\Psi_{k} = h_t(\Phi_{k}) = \Phi_{k}e^{-t\Lambda_{k}}\Phi_{k}^{\dagger}\Phi_{k} = \Phi_{k}e^{-t\Lambda_{k}}.
\end{equation} 

However, standard heat-kernel-based bases suffer from two primary limitations. First, reliance on a single diffusion time scale $t$ restricts the multi-scale representational capacity, thereby limiting the expressiveness of the resulting inhibition functions. Second, since the eigenvalues $\Lambda_k$ serve as fixed weights for the diffusion time, high-frequency components are often excessively suppressed regardless of their task-relevance. Ideally, frequency attenuation should be adaptively determined by the specific requirements of the downstream task. 

To address these shortcomings, we introduce a multi-scale eigenvalue-agnostic heat diffusion mechanism to generate optimized basis functions:
\begin{equation}\label{equ:heatnew}
\Psi_k = \Phi_{k}e^{-T},
\end{equation}
where $T = \text{diag}\{ t_1, t_2, \dots, t_k\}$ is a diagonal matrix of learnable parameters. The coefficients $\{t_i\}_{i=1}^k$ are initialized to zero, ensuring that the inhibition function acts as an identity mapping at the onset of training, thus preserving all original basis functions equally. Furthermore, we enforce parameter sharing for the diffusion matrix $T$ across the source and target domains, i.e., ${\Psi}_{\mathcal{X},k} = {\Phi}_{\mathcal{X},k} G$ and ${\Psi}_{\mathcal{Y},k} = {\Phi}_{\mathcal{Y},k} G$ with $G = e^{-T}$. This strategy not only curtails the computational overhead but also promotes spectral consistency between shapes, aligning with the robust optimization principles utilized in DiffusionNet~\cite{sharp2022diffusionnet}.

\subsection{Maps computation}\label{sec:axiomatic_functional_map_layer}
We utilize the softmax operator to efficiently generate a soft correspondence matrix~\cite{Eisenberger2021}, namely, 
\begin{equation}\label{eq: compute soft map}
    {\Pi}_{\mathcal{YX}} = \mathrm{Softmax}({F}_{\mathcal{Y}} {F}^{\top}_{\mathcal{X}}/\alpha),
\end{equation}
where $\alpha$ is the scaling factor to determine the softness of the correspondence matrix. 

The estimation of the functional map is a pivotal stage in the learning pipeline. Many existing SOTA frameworks~\cite{donati2020deep,Cao2023,bastian2024hybrid,luo2025deep} rely on linear system solvers~\cref{equ:desc and reg adj}, either independently or in conjunction with spectral projections~\cref{equ: compute C by Pi adj}, to ensure the properness of the resulting map. However, such solvers incur heavy computational costs and frequently suffer from numerical instability~\cite{magnet2024memory}. In contrast, we exclusively employ spectral projection~\cref{equ: compute C by Pi adj} to derive the advanced functional map. This integrated design leverages the formal equivalence between the two computational paradigms proven in~~\cref{sec:equ_sam_in_supp}, thereby making auxiliary solvers superfluous. By avoiding the complexities of least-squares optimization, our method achieves significantly lower computational overhead (see~\cref{sec:runtime}) and maintains high numerical robustness.

\subsection{Multi-resolution Unsupervised Spectral Loss}\label{sec:MSLoss}
To enforce structural consistency across pointwise maps, functional maps, and basis functions, we propose a novel multi-resolution unsupervised spectral loss $L_{mrs}$. Unlike conventional approaches that treat these components separately~\cite{NEURIPS2020_11953163}, our formulation unifies them within a single objective across multiple eigenvector resolutions:
\begin{equation}\label{equ:multi_spectral_loss}
    {L}_{mrs} = \sum^{k_{end}}_{k=k_{init}} \left\| {\Psi}_{\mathcal{Y},k} - {\Pi}_\mathcal{YX} {\Psi}_{\mathcal{X},k}  (C^{\mathrm{A}}_\mathcal{XY})^{\top} \right\|^{2}_\mathrm{F}.
\end{equation}

Remarkably, while contemporary SOTA methods~\cite{Cao2023,magnet2024memory,bastian2024hybrid,luo2025deep} necessitate complex combinations of multiple loss terms (e.g., bijectivity, orthogonality), our framework achieves superior performance using this single loss function alone, demonstrating its exceptional efficiency and robustness.

\subsection{Inference Stage}\label{sec:nn_inference} 
During inference, we first compute an initial pointwise map $\Pi_{\mathcal{YX}}$ via nearest neighbor search in the learned feature space:
\begin{equation}\label{equ: nnsearch1 in test}
    {\Pi}_{\mathcal{YX}} = \text{NS}({F}_\mathcal{Y}, {F}_\mathcal{X}).
\end{equation}
The final map $\Pi^{end}_{\mathcal{YX}}$ is then refined through functional map recovery:
\begin{equation}\label{equ: nnsearch2 in test}
    \Pi^{end}_{\mathcal{YX}} = \text{NS}(\Psi_{\mathcal{Y}, k_{end}}, \Psi_{\mathcal{X}, k_{end}}(C^{\mathrm{A}}_{\mathcal{XY}})^\top),
\end{equation}
where $C^{\mathrm{A}}_{\mathcal{XY}} = \Psi_{\mathcal{Y}, k_{end}}^\dagger \Pi_{\mathcal{YX}} \Psi_{\mathcal{X}, k_{end}}$. 

Unlike SOTA methods (e.g., ULRSSM~\cite{Cao2023}, HybridFMaps~\cite{bastian2024hybrid}) that switch between feature-based matching for non-isometric cases and functional map recovery for near-isometric ones, our approach maintains a unified pipeline (\cref{equ: nnsearch1 in test,equ: nnsearch2 in test}). This consistency stems from our learnable basis functions, which adaptively suppress spectral noise and mitigate non-isometric distortions via learned inhibition functions. Consequently, our optimized bases enable robust and direct pointwise map recovery across diverse matching scenarios.
\section{Experiments and Results}
\subsection{Datasets}\label{sec: dataset} 
We extensively evaluate our method on widely remeshed datasets, characterizing each dataset in terms of shape composition, deformation types, and training/testing settings, see~\cref{tab: dataset_desc}. 

\begin{table}[h!t]
\centering
\caption{Dataset Description. F and S denote the FAUST and SCAPE datasets, respectively. F\_a and S\_a refer to the anisotropic remeshed versions of FAUST and SCAPE. The terms near-iso and non-iso refer to near-isometric and non-isometric deformations, respectively. Aniso indicates different mesh connectivity, and topo-noise refers to topological noise.}
\scalebox{.8}{
\begin{tabular}{lrrr}
\toprule
\multicolumn{1}{l}{Dataset} & \multicolumn{1}{c}{Shape} & \multicolumn{1}{c}{Deformation} & \multicolumn{1}{c}{Training/Testing} \\ \midrule
F~\cite{ren2018continuous}                & human  & near-iso     & 80/20           \\
S~\cite{ren2018continuous}            & human      &  near-iso      & 51/20      \\
S19~\cite{melzi2019shrec}           & human    & near-iso      & -/44        \\
F\_a~\cite{donati2022deep}           & human   & aniso     & -/20        \\
S\_a~\cite{donati2022deep}              & human     & aniso     & -/20      \\
SMAL~\cite{zuffi2017}            & animals       &  non-iso    & 32/17     \\
DT4D-H~\cite{2022SmoothNonRigidShapeMatchingviaEffectiveDirichletEnergyOptimization}          & humanoid      & non-iso     & 198/95       \\
TOPKIDS~\cite{lahner2016shrec}         & human       & topo-noise  & 26/26       \\
\bottomrule
\end{tabular}}
\label{tab: dataset_desc}
\end{table}

\begin{table*}[h!t]
\centering
\caption{ Evaluating the matching results across various benchmarks, including near-isometric shape matching, cross-dataset generalization (FAUST, SCAPE, and SHREC'19), anisotropic meshing (F\_a and S\_a), and non-isometric shape matching (SMAL and DT4D-H), respectively.  The numbers in the table are mean geodesic errors~\cite{Kim2011BIM}, with all results multiplied by 100 for improved readability ($\times 100$). \textbf{Bold}: Best. \underline{Underline}: Runner-up. Due to the lack of aligned datasets in the SOTA method SmS~\cite{cao2024spectral} (only aligned FAUST is available in the open-source code), we are unable to conduct a comprehensive comparison with it and have therefore excluded it from the baselines.} 
\scalebox{0.8}{
\begin{tabular}{@{}c|lrrrrrrrrrrrr}
\toprule
    &Train             & \multicolumn{3}{c}{FAUST} & \multicolumn{3}{c}{SCAPE} & \multicolumn{2}{c}{FAUST} & \multicolumn{2}{c}{SCAPE} & \multicolumn{1}{c}{\multirow{2}{*}{SMAL}} & \multicolumn{1}{c}{{DT4D-H}}
\\ \cmidrule(lr){3-5} \cmidrule(lr){6-8} \cmidrule(lr){9-10} \cmidrule(lr){11-12} 
&Test            & F        & S   & S19    & F        & S   & S19    & F\_a      & S\_a  & F\_a    & S\_a  & \multicolumn{1}{c}{}  & \multicolumn{1}{c}{inter-class}                    
\\ \midrule
\parbox[t]{2mm}{\multirow{4}{*}{\rotatebox[origin=c]{90}{\textit{Axiomatic}}}}        
& ZoomOut~\cite{Melzi2019}             & 6.1         & 7.5   & -     & 6.1         & 7.5  & -  & 8.7         & 15.0        & 8.7         & 15.0  & 47.7          & 29.0      \\
&SmoothShells~\cite{eisenberger2020smooth}        & 2.5         & 4.7   & -     & 2.5         & 4.7  & -    & 5.4         & 5.0         & 5.4         & 5.0    & 34.9            & 6.3   \\
&DiscreteOp~\cite{ren2021discrete}          & 5.6         & 13.1   & -     & 5.6        & 13.1 & -  & 6.2         & 14.6       & 6.2       & 14.6   & 36.1       & 27.6         \\
&MWP~\cite{Hu2021}                 & 3.1         & 4.1    & -     & 3.1          & 4.1  & -     & 8.2         & 8.7       & 8.2         & 8.7    & 20.9       & 25.4        \\ \midrule

\parbox[t]{2mm}{\multirow{2}{*}{\rotatebox[origin=c]{90}{\textit{Sup.}}}}                  
&FMNet~\cite{litany2017deep}               & 11.1         & 30.0    & -     & 33.0          & 17.0  & -   & 42.0         & 43.0     & 43.0          & 41.0    & -      & 38.0    \\             
&GeomFmaps~\cite{donati2020deep}           & 2.6         & 3.4    & 9.9     & 3.0          & 3.0  & 12.2   & 3.2         & 3.8    & 8.4     & 3.1     & 4.3      & 4.1      \\  \midrule 
\parbox[t]{2mm}{\multirow{12}{*}{\rotatebox[origin=c]{90}{\textit{Unsupervised}}}}
                
&Deep Shells~\cite{eisenberger2020deep}         & 1.7         & 5.4  & 27.4     & 2.7         & 2.5    & 23.4   & 12.0        & 16.0       & 15.0         & 10.0    & 21.4           & 31.1  \\
&DUO-FMNet~\cite{donati2022deep}           & 2.5     & 4.2   & 6.4     & 2.7        & 2.6  & 8.4    & 3.0         & 4.4        & 3.1         & 2.7     & 4.8         & 15.8  \\
&AttentiveFMaps~\cite{li2022learning}      & 1.9         & 2.6  & 6.4      & 2.2        & 2.2     & 9.9  & 2.4         & \underline{2.8}        & \underline{2.5}         & 2.3   & 4.4                                       & 11.6        \\
&RFMNet~\cite{HU2023101189}               & 1.7         & \textbf{2.3}        & 6.3         & \underline{1.7}     & 2.1      & 6.9  & 3.6         & \textbf{2.6}        & 3.6         & 3.9       & 4.4                    & 13.9    \\
&ULRSSM~\cite{Cao2023}         & {1.6}         & 6.7   & 14.5     &  4.8       & \underline{1.9}   & 18.5   & 2.5         & 8.9        & 7.0          & {1.9}  & 4.5                                            & 5.2       \\
&ULRSSM(w.FT)~\cite{Cao2023}       & 1.6         & 2.2   & 5.7     & 1.6         & 1.9   & 6.7  & 1.9         & 2.4        & 2.1         & 1.9    & 4.2           & 4.1    \\
&DiffZO~\cite{magnet2024memory}                & 1.9    & \underline{2.4}   & \textbf{4.2}   & 1.9       & 2.4   & \underline{6.9}  & 2.2         & 3.8        & 2.7        & 2.4   & 4.3                        & 4.1    \\
&HybridFMaps~\cite{bastian2024hybrid} & \textbf{1.4}    & 4.2   & 9.5   & 2.3       & 1.8   & 13.0 & \underline{2.0}        &  4.6       & 3.4      & \textbf{1.8}   & \underline{3.5}           & \underline{3.9}      \\
&HybridFMaps(w.FT)~\cite{bastian2024hybrid} & 1.4   & 2.1   & 5.6   & 1.4       & 1.8   & 5.3  & 1.8         & 2.2        & 2.1       & 1.7   & 2.8                     & 3.5      \\
&DeepFAFM~\cite{luo2025deep}        & \underline{1.6}    & 2.7   & 7.0   & 1.9       & \textbf{1.9}   & 7.9  & 2.0         & 2.9        & 2.6        & \underline{1.9}   & 3.9          & 4.2      \\
&DenoiseFMaps~\cite{zhuravlev2025denoising}          & 1.8    & -   & -   & -      & 2.3   & - & \textbf{2.0}         & -       & -        & 2.3   & 4.3            & 12.8      \\
&\cellcolor{TRColor}{\textbf{Ours}}                & \cellcolor{TRColor}{1.8}    & \cellcolor{TRColor}{3.5}   & \cellcolor{TRColor}{\underline{6.2}}   & \cellcolor{TRColor}{\textbf{1.7}}       & \cellcolor{TRColor}{2.3}     & \cellcolor{TRColor}{\textbf{5.9}}  & \cellcolor{TRColor}{2.1}         & \cellcolor{TRColor}{3.6}       & \cellcolor{TRColor}{\textbf{2.1}}        & \cellcolor{TRColor}{2.3}  & \cellcolor{TRColor}{\textbf{2.6}}                                          & \cellcolor{TRColor}{\textbf{3.5}}     \\ 

\bottomrule
\end{tabular}}
\label{tab: near and non-iso}
\end{table*}

\subsection{Baselines}
We extensively compare our method with existing non-rigid deformable shape matching methods, which we categorize as follows:
\begin{itemize}
    \item \textit{Axiomatic approaches}, including ZoomOut \cite{melzi2019matching}, Smooth Shells \cite{eisenberger2020smooth}, DiscreteOp \cite{ren2021discrete}, and MWP \cite{Hu2021}.
    \item \textit{Supervised approaches}, FMNet~\cite{litany2017deep}, GeomFmaps~\cite{donati2020deep}.
    \item \textit{Unsupervised approaches}, 
    including UnsupFMNet~\cite{Halimi_2019}, SURFMNet~\cite{Roufosse_2019}, Deep Shells~\cite{eisenberger2020deep}, 
    NeuroMorph~\cite{eisenberger2021neuromorph}, 
    DUO-FMNet~\cite{donati2022deep},  AttentiveFMaps~\cite{li2022learning}, RFMNet~\cite{HU2023101189},
    ULRSSM~\cite{Cao2023}, DiffZO~\cite{magnet2024memory}, HybridFMaps~\cite{bastian2024hybrid}, DeepFAFM~\cite{luo2025deep}, and DenoiseFmaps~\cite{zhuravlev2025denoising}.     
\end{itemize}
Furthermore, ULRSSM~\cite{Cao2023} introduces a time-consuming fine-tuning technique, test-time adaptation, which adjusts network parameters for each test pair during inference. We extend ULRSSM~\cite{Cao2023} and HybridFMaps~\cite{bastian2024hybrid} by incorporating this technique, resulting in ULRSSM(w.FT) and HybridFMaps(w.FT). We then compare these extended versions, noting that we do not emphasize the fine-tuned results, as these methods require retraining on the test set, while other methods do not update parameters during inference. Using these fine-tuned methods as baselines highlights that our approach, which avoids fine-tuning, outperforms even the fine-tuned methods.

\subsection{Results}
We present extensive experimental results across multiple datasets, including challenging non-isometric ones. 

\textbf{Near-isometric matching}. The results of these benchmarks are provided in~\cref{tab: near and non-iso}, where our method is compared with current SOTA axiomatic, supervised, and unsupervised learning approaches. The results indicate that our method performs better than the previous SOTA axiomatic, supervised, and unsupervised methods, such as GeomFmaps and DiffZO on FAUST, and achieves comparable results on SCAPE. 

\textbf{Cross-dataset generalization.} To assess the generalization performance of our method, we train networks on remeshed datasets and test them on a different dataset, e.g., training on F and testing on S, and vice versa. We also use the challenging SHREC'19 dataset~\cite{melzi2019shrec} exclusively as a test set, providing a rigorous benchmark. As shown in~\cref{tab: near and non-iso}, the quantitative results demonstrate that our approach outperforms others in most scenarios. However, the SOTA unsupervised methods, ULRSSM~\cite{Cao2023} and HybridFMaps~\cite{bastian2024hybrid}, suffer significant performance drops on cross-dataset tasks (e.g., training on S and testing on S19), revealing their poor generalization ability. In contrast, our method consistently achieves superior performance, even surpassing the fine-tuned results of ULRSSM. These results clearly show that ULRSSM and HybridFMaps rely heavily on test-time adaptation, while our approach offers greater robustness.

\textbf{Matching with anisotropic meshing.} To evaluate robustness across different discretizations, we train networks on remeshed datasets and test them on anisotropic remeshed versions. The results presented in~\cref{tab: near and non-iso} show that our method achieves comparable performance in most settings and demonstrates greater resilience to changes in triangulation. For instance, when training on S and testing on F\_a, competitors like HybridFMaps experience significant performance declines, often overfitting to mesh connectivity and producing inaccurate predictions. In contrast, our method maintains strong robustness to varying mesh connectivity and consistently outperforms current methods.

\textbf{Non-isometric Shape Matching.} We evaluate our method on two challenging non-isometric datasets: SMAL~\cite{zuffi2017} and DT4D-H (inter-class)~\cite{2022SmoothNonRigidShapeMatchingviaEffectiveDirichletEnergyOptimization}, both featuring significant topological and geometric variations. We focus on the more demanding inter-class scenario of DT4D-H, as intra-class matching primarily involves near-isometric shapes where existing methods already saturate. As reported in~\cref{tab: near and non-iso}, our approach consistently outperforms all baselines across all non-isometric benchmarks. Notably, while HybridFMaps~\cite{bastian2024hybrid} enhances performance by incorporating extrinsic elastic bases~\cite{2023ElasticBasis}, our method, relying solely on optimized intrinsic Laplacian bases, surpasses even its fine-tuned results on SMAL and DT4D-H. These results demonstrate that our learnable inhibition functions effectively compensate for non-isometric distortions, providing a robust and superior alternative to existing intrinsic (e.g., ULRSSM~\cite{Cao2023}, DeepFAFM~\cite{luo2025deep}) and extrinsic-aided frameworks.

\section{Conclusion and Limitation}
We introduce a novel approach that unifies feature learning and basis optimization in an end-to-end manner, built upon the proposed advanced functional map framework. Our method outperforms existing feature learning-based functional map approaches in various challenging scenarios while maintaining computational efficiency. Theoretically, we demonstrate that basis learning is equivalent to spectral convolution, with inhibition functions acting as filters, paving the way for new directions in shape matching. 

Despite being broadly effective, we acknowledge its limitations under extreme conditions, such as severe non-isometric deformations and partial shape matching, which heavily distort the structural information encoded in the Laplacian basis. Employing deformation‐aware representations or incorporating extrinsic information offers viable avenues for addressing these challenges. 
\section{Acknowledgments}
This work is supported by National Natural Science Foundation of China under Grant 62271452. 

{\small
\bibliographystyle{ieeenat_fullname}
\bibliography{main}
}
\clearpage
\setcounter{page}{1}

\maketitlesupplementary
In this supplementary document, we first review the theoretical background of spectral graph convolutional networks and heat diffusion in~\cref{sec:gnns_heat}. In~\cref{sec:the_pro}, we provide formal proofs for our advanced functional map framework and the G-ZoomOut refinement algorithm. ~\cref{sec:results} presents comprehensive experiments to evaluate the robustness of our method, covering implementation details, topological noise matching, runtime comparisons, ablation studies, and parameter analysis. ~\cref{sec:Inhibition_analysis} offers an in-depth analysis of the learned inhibition functions across various datasets, followed by additional qualitative visualizations in~\cref{sec:visualization}. Finally, in~\cref{sec:limitation}, we outline promising avenues for future research.

\section{Background: Spectral Graph Convolution Networks and Heat Diffusion}\label{sec:gnns_heat}
In the main manuscript, our basis function learning framework draws upon the principles of spectral graph convolutions~\cite{defferrard2016ChebyNet} and heat diffusion~\cite{sun2009concise}. Due to space constraints, we provide a detailed overview of their theoretical foundations and a comparative analysis of their distinctions in this section. Readers seeking a more exhaustive treatment are referred to existing surveys on spectral geometry processing~\cite{wu2020comprehensive,shuman2016vertex}.

\paragraph{Spectral convolution}
Spectral Graph Convolutional Neural Networks (Spectral GCNs) represent a cornerstone in Graph Signal Processing (GSP)~\cite{ortega2018graph}. By generalizing the classical convolution operator from Euclidean domains to non-Euclidean graph structures~\cite{defferrard2016ChebyNet}, Spectral GCNs enable the effective extraction of local topological features. Methodologically, they follow a transform-filter-recover paradigm: a graph signal $x_{in}$ is first projected into the spectral domain via the Laplacian eigensystem; subsequently, a filter function $f$ is applied to modulate the spectral components; finally, the filtered signal is back-projected to the spatial domain, namely,

\begin{equation}\label{equ:Spectral_Convolution_continue_in_supp}
x_{out} = f*x_{in} = \mathcal{T}^{-1}(\mathcal{T}(f) \cdot \mathcal{T}({x_{in}})),
\end{equation}
where $\cdot$ is the matrix multiplication operator, $\mathcal{T}$ and $\mathcal{T}^{-1}$ denote the Fourier transform and corresponding inverse transform, respectively. 

In the discrete setting, the operation is formulated as:
\begin{equation}\label{equ:SpectralGraphConv_in_supp}
x_{out} = \Phi f(\Lambda)\Phi^{\dagger}x_{in}.
\end{equation}
Notably, while $\Phi^{\dagger} = \Phi^{\top}$ for orthonormal spectral bases in standard graph convolution, we employ the pseudoinverse notation $\dagger$ to maintain a consistent mathematical formalism throughout this work.

\paragraph{Heat diffusion}
The heat diffusion equation governs the temporal evolution and smoothing of a signal $u(s,t)$ over a geometric domain. Formally, 
\begin{equation}
    \frac{\partial u}{\partial t} =  \Delta u,
\end{equation}
where $s$ and $t$ denote the spatial position and diffusion time, respectively, while $\Delta$ signifies the Laplacian operator. 

As previously discussed, the solution to this continuous partial differential equation admits a closed-form expression via spectral decomposition. Given an initial signal distribution $u(s, 0) = x_{in}$, the solution at time $t$ can be formally represented as:
\begin{equation}\label{equ:heat_continue_in_supp}
u(s, t)= h_t(x_{in}) = e^{-t\Delta} x_{in},
\end{equation}
where $e^{-t\Delta}$ is the heat operator $h_t$. By expanding this solution in the spectral basis~\cite{sharp2022diffusionnet}, we recover the formulation presented in~\cref{equ:heat_continue_in_supp}:
\begin{equation}
x_{out} = h_t(x_{in})= \Phi e^{-t\Lambda} \Phi^\dagger x_{in}.
\end{equation}

In the spectral domain, the solution reveals that heat diffusion acts as an isotropic low-pass filter. The exponential term $e^{-t\lambda}$ induces a decay of high-frequency components, where the time parameter $t$ controls the diffusion scale: small $t$ preserves local geometric nuances, while large $t$ emphasizes global topological features.

\begin{table*}[h!t]
\centering
\caption{Comparison between Spectral Convolutional Neural Networks and Heat Diffusion.}
\label{tab:spectral_vs_heat}
\scalebox{0.85}{
\begin{tabular}{@{}lll@{}}
\toprule
{Property} & {Spectral CNNs} & {Heat Diffusion} \\ \midrule
{Filter Form $f(\lambda)$} & Polynomials ($\sum^{J}_{j=0} \theta_j p_j(\lambda)$) & Exponential ($e^{-t\lambda}$) \\
{Spectral Response} & Flexible (low-pass, high-pass, band-pass) & Strictly low-pass (exponential decay) \\
{Parameter Role} & Learnable weights $\theta$ for task-adaptation & Diffusion time $t$ for physical scaling \\
{Spatial Locality} & Controlled by polynomial order $J$ & Governed by diffusion time $t$ \\
{Stability} & Sensitive to basis drift and noise & Robust to isometric deformations \\
\bottomrule
\end{tabular}}
\end{table*}

The key differences between spectral CNNs and heat diffusion are summarized in~\cref{tab:spectral_vs_heat}. While Spectral CNNs offer flexibility, they often suffer from basis drift. Heat diffusion provides physical stability but lacks task-specific adaptability. Our proposed learnable inhibition function $G = e^{-T}$ bridges this gap by retaining the stable exponential form while allowing for eigenvalue-agnostic, multi-scale optimization.

\section{Theoretical Proofs}\label{sec:the_pro} 
We give the proof in the Supplementary Material for completeness, including 1) formulation of the optimization problem for Advanced Functional Maps; 2) equivalence between the two computational approaches under certain conditions; 3) a variant of the orthogonality-preserving energy; and 4) G-ZoomOut: a spectral upsampling iterative algorithm defined on the learnable spectral base.

\subsection{Optimization of advanced functional maps}\label{sec:desc_pre_lam_commu_in_supp}

\begin{unnumberedthm}
The Advanced Functional Maps optimization problem can be formulated as:
\begin{small}
\begin{equation*}\label{equ:desc and reg adj in app}
   {C}^{\mathrm{A}}_{\mathcal{XY}} = \mathop{\arg\min}\limits_{ {C}^{\mathrm{A}}_{\mathcal{XY}}\!} \left\|  {C}^{\mathrm{A}}_{\mathcal{XY}} \Psi_{\mathcal{X},k}^{\dagger}   F_{\mathcal{X}} - \Psi_{\mathcal{Y},k}^{\dagger}   F_{\mathcal{Y}} \right\|^{2}_{\mathrm{F}} 
    + \lambda {E}_{reg}({C}^{\mathrm{A}}_{\mathcal{XY}}), 
\end{equation*} 
\end{small}
where $E_{reg} = \left\| C^{\mathrm{A}}_{\mathcal{XY}}\Lambda_{\mathcal{X}} - \Lambda_{\mathcal{Y}} C^{\mathrm{A}}_{\mathcal{XY}} \right\|^{2}_{\mathrm{F}}$.
\end{unnumberedthm}
\begin{proof}
Clearly, the features $F_{\mathcal{X}}$ and $ F_{\mathcal{Y}}$ can be linearly represented by the learnable Laplacian basis functions $\Psi_{\mathcal{X},k}$ and $\Psi_{\mathcal{Y},k}$ respectively; the Advanced Functional Maps ${C}^{\mathrm{A}}_{\mathcal{XY}}$ then correspond to the transfer matrix between their Fourier coefficients, satisfying, 
\begin{equation*}
  {C}^{\mathrm{A}}_{\mathcal{XY}} \Psi_{\mathcal{Y},k}^{\dagger}   F_{\mathcal{Y}} = \Psi_{\mathcal{X},k}^{\dagger}   F_{\mathcal{X}} .
\end{equation*} 
On the other hand, according to Laplacian commutativity constraint term~\cite{Ovsjanikov2012}, we have
\begin{equation*}
    T_F\circ L_{\mathcal{X}} = L_{\mathcal{Y}} \circ T_F. 
\end{equation*}
In discrete triangle mesh settings, $L \to \Phi_{k} \Lambda \Phi_{k}^{\dagger}$ and $T_F \to \Psi_{\mathcal{Y},k} {C}^{\mathrm{A}}_{\mathcal{XY}} \Psi_{\mathcal{X},k}^{\dagger}$, we have
\begin{small}
\begin{equation*}
\begin{aligned}
\Psi_{\mathcal{Y},k} {C}^{\mathrm{A}}_{\mathcal{XY}} \Psi_{\mathcal{X},k}^{\dagger} \Phi_{\mathcal{X},k} \Lambda_{\mathcal{X}} \Phi_{\mathcal{X},k}^{\dagger} &= \Phi_{\mathcal{Y},k} \Lambda_{\mathcal{Y}} \Phi_{\mathcal{Y},k}^{\dagger} \Psi_{\mathcal{Y},k} {C}^{\mathrm{A}}_{\mathcal{XY}} \Psi_{\mathcal{X},k}^{\dagger}.
\end{aligned}
\end{equation*}
\end{small}
Since $ \Psi_{\mathcal{X},k}^{\dagger}= G^{-1}_{\mathcal{X}} \Phi_{\mathcal{X},k}^{\dagger}$, $\Psi_{\mathcal{Y},k} = \Phi_{\mathcal{Y},k} G_{\mathcal{Y}}$, and $\Phi_{\cdot, k}^{\dagger} \Phi_{\cdot, k} =I$, then  
\begin{equation*}
\begin{aligned}
\Psi_{\mathcal{Y},k} {C}^{\mathrm{A}}_{\mathcal{XY}} G^{-1}_{\mathcal{X}}\Lambda_{\mathcal{X}} \Phi_{\mathcal{X},k}^{\dagger} &= \Phi_{\mathcal{Y},k} \Lambda_{\mathcal{Y}} G_{\mathcal{Y}} {C}^{\mathrm{A}}_{\mathcal{XY}} \Psi_{\mathcal{X},k}^{\dagger}.
\end{aligned}
\end{equation*}
Moreover, $\Lambda_{\mathcal{X}}$, $\Lambda_{\mathcal{Y}}$,  $G^{-1}_{\mathcal{X}}$, and $G_{\mathcal{Y}}$ are diagonal matrices, their matrix multiplication is commutative, then
\begin{equation*}
\Psi_{\mathcal{Y},k} {C}^{\mathrm{A}}_{\mathcal{XY}} \Lambda_{\mathcal{X}} \Psi_{\mathcal{X},k}^{\dagger} = \Psi_{\mathcal{Y},k}\Lambda_{\mathcal{Y}} {C}^{\mathrm{A}}_{\mathcal{XY}} \Psi_{\mathcal{X},k}^{\dagger}.
\end{equation*}
Finally, multiplying both sides of the equation by $\Psi_\mathcal{Y}^\dagger$ on the left and $\Psi_\mathcal{X}$ on the right, we obtain:
\begin{equation*}
 {C}^{\mathrm{A}}_{\mathcal{XY}} \Lambda_{\mathcal{X}}  = \Lambda_{\mathcal{Y}} {C}^{\mathrm{A}}_{\mathcal{XY}}.
\end{equation*}
\end{proof}

\subsection{Equivalent computation}\label{sec:equ_sam_in_supp}
\begin{unnumberedthm}
The two computational methods are equivalent, i.e, 
\begin{equation*}\label{equ: compute C by Pi adj in app}{C}^{\mathrm{A}}_\mathcal{XY}=\Psi_{\mathcal{Y}, k}^\dagger\Pi_\mathcal{YX}\Psi_{\mathcal{X},k}
\end{equation*}
is equal to ${C}^{\mathrm{A}}_\mathcal{XY}$ in~\cref{equ:desc and reg adj in app}, if conditions (a)-(c)~\cite{cao2024revisiting} hold:
\begin{itemize}
    \item [] (a) $\Pi_{\mathcal{YX}}F_{\mathcal{X}}=F_\mathcal{Y}$.
    \item [] (b) $F_{\mathcal{X}}$ and $F_\mathcal{Y}$ are both in the span of $\Phi_{\mathcal{X},k}$ and $\Phi_{\mathcal{Y},k}$, respectively.
    \item [] (c) $\Phi_{\cdot,k}^{\dagger} F_{\cdot}\in \mathbb{R}^{k\times d}$ is full rank and $\lambda = 0$ in~\cref{equ:desc and reg adj in app}.
\end{itemize}
\end{unnumberedthm}

\begin{proof}
By condition (b) and $F_{\cdot} = \Phi_{\cdot, k} \Phi_{\cdot, k}^{\dagger} F_{\cdot}$, where $\Phi_{\cdot, k}$ is the \textit{inverse} Fourier transform, then we rewrite condition (a) as
\begin{equation*}\label{eq: Functional maps in Lap}
 {\Pi}_{\mathcal{YX}} \Phi_{\mathcal{X},k} \Phi_{\mathcal{X}, k}^{\dagger} F_{\mathcal{X}}  = \Phi_{\mathcal{Y},k} \Phi_{\mathcal{Y}, k}^{\dagger} F_{\mathcal{Y}}.
\end{equation*}
Since $\Phi_{\cdot,k} \Phi_{\cdot, k}^{\dagger} = \Phi_{\cdot,k} G G^{-1}\Phi_{\cdot, k}^{\dagger}= \Psi_{\cdot,k} \Psi_{\cdot, k}^{\dagger}$, then 
\begin{equation*}
 {\Pi}_{\mathcal{YX}} \Psi_{\mathcal{X},k} \Psi_{\mathcal{X}, k}^{\dagger} F_{\mathcal{X}}  = \Psi_{\mathcal{Y},k} \Psi_{\mathcal{Y}, k}^{\dagger} F_{\mathcal{Y}}.
\end{equation*}
Multiplying the equation by $ \Psi_{\mathcal{Y}, k}^{\dagger}$, we have
\begin{equation*}
 \Psi_{\mathcal{Y}, k}^{\dagger}{\Pi}_{\mathcal{YX}} \Psi_{\mathcal{X},k} \Psi_{\mathcal{X}, k}^{\dagger} F_{\mathcal{X}}  = \Psi_{\mathcal{Y}, k}^{\dagger} F_{\mathcal{Y}}.
\end{equation*}
On the other hand, from~\cref{equ:desc and reg adj in app} with $\lambda =0$, we have ${C}^{\mathrm{A}}_\mathcal{XY} \Psi_{\mathcal{X}, k}^{\dagger} F_{\mathcal{X}} = \Psi_{\mathcal{Y}, k}^{\dagger} F_{\mathcal{Y}}$. Meanwhile, $\Phi_{\cdot,k}^{\dagger} F_{\cdot}\in \mathbb{R}^{k\times d}$ is full rank by condition (c), so that the solution from~\cref{equ:desc and reg adj in app} and~\cref{equ: compute C by Pi adj in app} is unique.
\end{proof}


Interestingly, the validity of the above derivations does not depend on the inhibition functions. This implies that the properties of functional maps that hold for the original basis functions also hold for the learned basis functions.

\subsection{Enforcing orthogonality of the advanced functional map}\label{sec:nn_FMsolver}


\begin{unnumberedthm}
An alternative formulation of the orthogonality preservation is derived, namely, 
\begin{equation*}\label{equ:area preservation variation adj in app}
E({\Pi}_{\mathcal{YX}},C^{\mathrm{A}}_\mathcal{XY},G) = \left\|\Psi_{\mathcal{Y},k} - {\Pi}_{\mathcal{YX}}\Psi_{\mathcal{X},k} (C^{\mathrm{A}}_\mathcal{XY})^{\top}\right\|^{2}_{\mathrm{F}}, 
\end{equation*}
which not only enforces orthogonality but also penalizes the image of $\Pi_\mathcal{YX}$ lying outside $\Psi_\mathcal{Y}$ via a regularizer, i.e, 
\begin{small}
\begin{equation*}\label{eq:add_reg}
\begin{aligned}
    &\mathcal{R}(\Pi_\mathcal{YX}, C^{\mathrm{A}}_\mathcal{XY}) =\\ &\left\| \Bigl( I- \Psi_{\mathcal{Y},k} \Psi^{\dagger}_{\mathcal{Y},k} \Bigr) \Bigl(\Psi_{\mathcal{Y},k} - {\Pi}_{\mathcal{YX}}\Psi_{\mathcal{X},k} (C^{\mathrm{A}}_\mathcal{XY})^{\top}\Bigr) \right\|^2_{M_{\mathcal{Y}}}.
\end{aligned}
\end{equation*}
\end{small}
\end{unnumberedthm}

\begin{proof} To begin with, we show that \scalebox{0.90}{
$\left\| X \right\|^2_{M_{\mathcal{Y}}} =
\left\| G_\mathcal{Y} \Psi^{\dagger}_{\mathcal{Y},k} X\right\|^2_{\mathrm{F}} + \left\| \Bigl( I- \Psi_{\mathcal{Y},k} \Psi^{\dagger}_{\mathcal{Y},k} \Bigr) X \right\|^2_{M_{\mathcal{Y}}}$}, where $X = \Psi_{\mathcal{Y},k} - {\Pi}_{\mathcal{YX}}\Psi_{\mathcal{X},k} (C^{\mathrm{A}}_\mathcal{XY})^{\top}$.

On the left, $\left\| X \right\|^2_{M_{\mathcal{Y}}} = tr(X^{\top} M_{\mathcal{Y}}X)$. 

For the first term on the right side, \scalebox{0.90}{ $\left\| G_\mathcal{Y} \Psi^{\dagger}_{\mathcal{Y},k} X\right\|^2_{\mathrm{F}} = tr \Bigl(  (G_\mathcal{Y} \Psi^{\dagger}_{\mathcal{Y},k} X)^{\top} G_\mathcal{Y} \Psi^{\dagger}_{\mathcal{Y},k} X\Bigl) $},
since $ \Psi^{\dagger}_{\mathcal{Y},k} = G^{-2}_{\mathcal{Y}} \Psi^\top_{\mathcal{Y},k}M_{\mathcal{Y}}$.
we have $\left\| G_\mathcal{Y} \Psi^{\dagger}_{\mathcal{Y},k} X\right\|^2_{\mathrm{F}} = tr \Bigl( X^{\top} M_\mathcal{Y} \Psi_{\mathcal{Y},k} \Psi^{\dagger}_{\mathcal{Y},k} X\Bigl)$. 

For the second term on the right side, $\left\| \Bigl( I- \Psi_{\mathcal{Y},k} \Psi^{\dagger}_{\mathcal{Y},k} \Bigr) X \right\|^2_{M_{\mathcal{Y}}}= tr \Bigl( X^{\top} (I-\Psi_{\mathcal{Y},k} \Psi^{\dagger}_{\mathcal{Y},k})^{\top} M_\mathcal{Y} (I-\Psi_{\mathcal{Y},k} \Psi^{\dagger}_{\mathcal{Y},k}) X \Bigl) = tr \Bigl(X^{\top}(I- M_{\mathcal{Y}} \Psi_{\mathcal{Y},k} G^{-2}_\mathcal{Y} \Psi^{\top}_{\mathcal{Y},k} ) M_{\mathcal{Y}} (I-  \Psi_{\mathcal{Y},k} G^{-2}_\mathcal{Y}  \Psi^{\top}_{\mathcal{Y},k} M_{\mathcal{Y}} )   X  \Bigl) = tr \Bigl( X^{\top} (  M_{\mathcal{Y}} -  M_{\mathcal{Y}}\Psi_{\mathcal{Y},k} \Psi^{\dagger}_{\mathcal{Y},k} )  X  \Bigl)$. 

Then, we have $\arg\min_{\Pi_{\mathcal{YX}}} \left\| X \right\|^2_{M_{\mathcal{Y}}} = \arg\min_{\Pi_{\mathcal{YX}}} \left\| G_\mathcal{Y} \Psi^{\dagger}_{\mathcal{Y},k} X\right\|^2_{\mathrm{F}} + \left\| \Bigl( I- \Psi_{\mathcal{Y},k} \Psi^{\dagger}_{\mathcal{Y},k} \Bigr) X \right\|^2_{M_{\mathcal{Y}}} $ with $X = \Psi_{\mathcal{Y},k} - {\Pi}_{\mathcal{YX}}\Psi_{\mathcal{X},k} (C^{\mathrm{A}}_\mathcal{XY})^{\top}$.

Clear, $\arg\min_{\Pi_{\mathcal{YX}}} \left\| G_\mathcal{Y} \Psi^{\dagger}_{\mathcal{Y},k} X\right\|^2_{\mathrm{F}} = \arg\min_{\Pi_{\mathcal{YX}}} \left\|  \Psi^{\dagger}_{\mathcal{Y},k} X\right\|^2_{\mathrm{F}}$ since $\Psi_{\mathcal{Y}, k}^{\dagger}{\Pi}_{\mathcal{YX}} \Psi_{\mathcal{X},k} (C^{\mathrm{A}}_\mathcal{XY})^{\top}={I}$ is equal to $G_\mathcal{Y}\Psi_{\mathcal{Y}, k}^{\dagger}{\Pi}_{\mathcal{YX}} \Psi_{\mathcal{X},k} (C^{\mathrm{A}}_\mathcal{XY})^{\top}={G_\mathcal{Y}}$ for a pointwisp map ${\Pi}_{\mathcal{YX}}$, with $G_\mathcal{Y}$ is a diagonal full rank matrix. Moreover we have $\arg\min_{\Pi_{\mathcal{YX}}} \left\| X \right\|^2_{M_{\mathcal{Y}}} = \arg\min_{\Pi_{\mathcal{YX}}} \left\| X \right\|^2_{F}$.
\end{proof}


\paragraph{Refinement}
We propose a spectral upsampling alternating refinement algorithm, dubbed G-ZoomOut. See~\cref{alg: G-ZoomOut refinement}.
\begin{algorithm}
\caption{: G-ZoomOut refinement}
\label{alg: G-ZoomOut refinement}
\begin{algorithmic}[1]
\State \textbf{Input:} $k_{init}$, $k_{end}$, and  $\Pi^{init}_\mathcal{YX}$
\State \textbf{Output:} $\Pi_\mathcal{YX}$
\For{$k = k_{init},k_{init}+1,\cdots, k_{end} $}
    \State ${C}^{\mathrm{A}}_\mathcal{XY}=\Psi_{\mathcal{Y},k}^\dagger\Pi_\mathcal{YX}\Psi_{\mathcal{X},k}$
    \State $\Pi_\mathcal{YX}=\text{NS}(\Psi_{\mathcal{Y},k},\Psi_{\mathcal{X},k}({C^{\mathrm{A}}_\mathcal{XY}})^{\top})$  
\EndFor
\end{algorithmic}
\end{algorithm}

In summary, the advanced functional map framework provides a systematic theoretical grounding for basis function optimization while preserving full compatibility with conventional functional map representations. This principled formulation allows immediate deployment through standard functional map pipelines without requiring structural modifications.

\section{Experiments and Results}\label{sec:results}
\subsection{Implementation details} 
For feature learning, we use DiffusionNet~\cite{sharp2022diffusionnet} as the feature extractor with its default settings, which employs 16-dimensional HKS features~\cite{sun2009concise} as input and generates 128-dimensional learned features for the network. For basis function learning, \textit{only} 200 diffusion times are required for the proposed heat convolution network, which is identical to the maximum number of truncated eigenvectors. For pointwise map computation, we set $\alpha=0.07$. For functional map computation, we use truncated eigensystems with $k=\{100,150,200\}$, where ${k_{init}=100}$ and ${k_{end}=200}$. For training, we use the Adam optimizer~\cite{Kingma15} with a learning rate of $0.001$ for all learning parameters. 

\begin{table}[h!t]
\centering
\caption{Topological noise on TOPKIDS\cite{lahner2016shrec}. The table presents the mean geodesic errors ($\times 100$). \textbf{Bold}: Best. \underline{Underline}: Runner-up.}
\scalebox{.9}{
\begin{tabular}{lrc}
\hline
                    & \multicolumn{1}{c}{TOPKIDS} & \multicolumn{1}{c}{Fully intrinsic} \\ \hline
                    \multicolumn{3}{c}{Axiomatic Methods}    \\
ZoomOut~\cite{Melzi2019}             & 33.7        &  {\color{green}\ding{51}}    \\
SmoothShells~\cite{eisenberger2020smooth}        & 11.8        & {\color{red}\ding{55}}             \\
DiscreteOp~\cite{ren2021discrete}           & 35.5        & {\color{green}\ding{51}}              \\
\hline
                    \multicolumn{3}{c}{Unsupervised Methods}     \\
UnsupFMNet~\cite{Halimi_2019}          & 38.5        & {\color{green}\ding{51}}  \\
SURFMNet~\cite{Roufosse_2019}            & 48.6        & {\color{green}\ding{51}}               \\
Deep Shells~\cite{eisenberger2020deep}         & 13.7        & {\color{red}\ding{55}}             \\
NeuroMorph~\cite{Eisenberger2021}          & 13.8        & {\color{red}\ding{55}}           \\
AttentiveFMaps~\cite{li2022learning}      & 23.4        & {\color{green}\ding{51}}              \\
ULRSSM~\cite{Cao2023}              & 9.4   & {\color{green}\ding{51}}         \\
ULRSSM(w.FT)~\cite{Cao2023}        & 9.2         & {\color{green}\ding{51}}     \\
HybridFMaps~\cite{bastian2024hybrid}             & \textbf{5.0}   & {\color{red}\ding{55}}         \\
HybridFMaps (w.FT)~\cite{bastian2024hybrid}        & 5.1         & {\color{red}\ding{55}}    \\
DeepFAFM~\cite{luo2025deep}            & {6.3}         & {\color{green}\ding{51}}     \\
DenoisFMaps~\cite{zhuravlev2025denoising}               & 43.6   & {\color{green}\ding{51}}    \\
\cellcolor{TRColor}{\textbf{Ours}}                & \cellcolor{TRColor}{\underline{5.9}}  & \cellcolor{TRColor}{{\color{green}\ding{51}}}  \\
\hline
\end{tabular}}
\label{tab: topo}
\end{table}

\subsection{Matching with topological noise} 
To evaluate robustness against topological perturbations, we conducted experiments on the SHREC'16 TOPKIDS dataset~\cite{lahner2016shrec}, which contains 26 shapes including one pristine model and 25 variants with varying topological noise levels. The quantitative results are summarized in~\cref{tab: topo}. Since topological noise distorts the intrinsic geometry of shapes, many purely intrinsic methods, such as AttentionFMaps~\cite{li2022learning} and DenoiseFMaps~\cite{zhuravlev2025denoising}, struggle to produce satisfactory results under such conditions. In contrast, methods that incorporate extrinsic information, such as DeepShell and HybridFMaps, often demonstrate improved robustness and accuracy. However, our approach outperforms all purely intrinsic methods, including the current leading techniques such as ULRSSM~\cite{Cao2023}, fine-tuned ULRSSM, and DeepFAFM~\cite{luo2025deep}. In contrast, HybridFMaps achieves notable performance by incorporating elastic basis functions~\cite{2023ElasticBasis} that integrate extrinsic information into ULRSSM. A comparative analysis shows that these improvements in HybridFMaps heavily rely on extrinsic cues. On the other hand, our purely intrinsic method delivers performance competitive with HybridFMaps, highlighting the superior effectiveness of our approach

\begin{table}[h!t]
\centering
\caption{Runtime comparison in seconds with different vertex resolutions. Fmap computation manner: Average runtime per method (100 iterations); Training: Average training time per method (100 iterations). Inference: Average inference time per method (100 shape pairs). All the statistics were collected on a server with Intel(R) Xeon(R) Platinum 8358 CPU @ 2.60GHz, and a single NVIDIA A100-958 SXM4-80GB GPU.}
\scalebox{.7}{
\begin{tabular}{lrrrrr}
\hline
Method                    & \multicolumn{1}{r}{5K} & \multicolumn{1}{r}{8K} & \multicolumn{1}{r}{10K}& \multicolumn{1}{r}{12K}& \multicolumn{1}{r}{15K}\\ \hline
                   \multicolumn{4}{r}{Fmap computation manner}     \\
$a$. Fmap solver             & 21.94       & 22.30  & 22.44  & 22.64  & 22.77              \\
$b$. Softmax+spec proj          & 0.13       & 0.28  & 0.41  & 0.57  & 0.87               \\

\hline
                    \multicolumn{4}{r}{Training}     \\
ULRSSM$^{a,b}$               & 66.43       & 76.95  & 87.90  & 99.10  & 121.92      \\
DiffZO                 & 24.81       & 37.92  & 45.83  & 54.67  & 71.21 \\
DeepFAFM$^{a,b}$            & 68.73      & 77.78  & 88.23  & 99.29  & 122.10     \\
Ours$^{b}$              & 17.57       & 28.17  & 38.95  & 49.57  & 67.60   \\
\hline
                    \multicolumn{4}{r}{Inference}     \\
ULRSSM               & 15.52      & 28.30  & 42.37  & 56.45  & 87.35      \\
ULRSSM(w.fineturn)        & 924.29       & 1029.38  & 1126.26  & 1236.79  & 1422.40     \\
DiffZO               & 49.03     & 56.92  & 61.25  & 67.16  & 78.40 \\
DeepFAFM            & 20.36     & 34.10  & 47.64  & 63.96  & 92.25     \\
Ours              & 14.37      & 28.60  & 43.48  & 60.08  & 91.98   \\
\hline
\end{tabular}}
\label{tab: runtime}
\end{table}
\subsection{Runtime comparison}\label{sec:runtime}
We evaluate the training and inference time of our method across datasets with varying vertex resolutions, and additionally report the computational costs of two prevalent functional map computation methods: the functional map solver$^{a}$ and the softmax-based spectral projector$^{b}$. All competitors are implemented using their original parameter configurations as reported in respective publications. The quantitative results are summarized in~\cref{tab: runtime}. 

Our method demonstrates the shortest training time, exhibiting over twice the speed of competing approaches like ULRSSM. Since both ULRSSM and DeepFAFM employ two computational approaches$^{a,b}$  bidirectionally for functional maps computation, where the functional map solver method$^{a}$ proves significantly time-consuming (see top section of~\cref{tab: runtime}), their training processes demand substantial computational time. While DiffZO currently stands as the most efficient functional maps learning method by eliminating time-consuming functional maps solvers (alongside our approach), it incurs additional computational overhead through implicit iterative solving of both pointwise and functional maps. In contrast, our framework achieves exceptional computational efficiency via the softmax-based spectral projector$^b$, contributing substantially to the superior training efficiency of our approach.

In inference time comparisons, our method achieves results comparable to ULRSSM and DeepFAFM, as all three approaches share broadly similar computational pipelines. In contrast, DiffZO incurs significantly longer inference times due to its iterative ZoomOut-style upsampling. While its implicit representation benefits from the KeOps acceleration library~\cite{charlier2021kernel}, yielding superior efficiency on high-resolution meshes (e.g., 15K vertices), this does not compensate for its overall latency. Notably, ULRSSM's fine‑tuning strategy improves accuracy at the cost of a 20× inference-time increase. Crucially, without employing any fine-tuning, our method surpasses even the fine‑tuned variant of ULRSSM in final accuracy, demonstrating its efficiency and effectiveness jointly.

\subsection{Ablation Studies}\label{sec: abl}
We conduct ablation studies across multiple datasets, covering non-isometric and topological noise scenarios, to evaluate the performance of our proposed components: basis learning and the multi-scale spectral loss (MSLoss).

The results are presented in~\cref{tab: ablation studies}. By comparing the first and last rows, we see that both the basis learning module and the MSLoss term contribute significantly to improving matching performance, primarily by enhancing spectral basis representations across diverse settings. When comparing the second and last rows, it becomes clear that the basis learning module plays a critical role in the method's success. Its absence leads to a noticeable drop in performance across various matching tasks. Lastly, comparing the third and last rows reveals that, while MSLoss only slightly improves non-isometric matching accuracy, it substantially enhances robustness against topological noise.

\begin{table}[h!t]
\centering
\caption{We conduct ablation studies on remeshed FAUST, SHREC'19, DT4D-H, and TOPKIDS using two ablated configurations: "w.o. basis learning" removes basis optimization by fixing $G \equiv I$ (reverting to handcrafted Laplacian basis), while "w.o. MSLoss" uses single-resolution loss at $k = 200$.}
\scalebox{0.85}{
\begin{tabular}{lrr}
\hline
\multicolumn{1}{l}{Settings}            &  \multicolumn{1}{c}{DT4D-H inter} & \multicolumn{1}{c}{TOPKIDS} \\ \hline
w.o basis learning and MSLoss      & 8.1  & 19.7  \\ 
w.o basis learning        & 4.9  & 16.6  \\ 
w.o MSLoss         & 3.7  & 10.3  \\ 
Ours                & 3.5   & 5.9   \\ 
\hline
\end{tabular}}
\label{tab: ablation studies}
\end{table}

\subsection{Parameter analysis} 
Due to our lightweight design, our method requires only a minimal set of manually configured hyperparameters: the number of Laplacian basis functions, diffusion times, and the scaling factor $\alpha$. Since our multi-resolution loss already incorporates sampling across resolutions, and the number of diffusion times matches the maximum number of Laplacian basis functions. We only present the impact of the scaling factor $\alpha$ on our method's performance. The results are summarized in~\cref{tab: parameter analysis}. The results indicate that the accuracy of our method exhibits only minor variations with changes in $\alpha$, strongly demonstrating its robustness. Consequently, we set $\alpha=0.07$, aligning with the configuration used in methods such as ULRSSM, HybridFMaps, and DeepFAFM.

\begin{table}[h!t]
\centering
\caption{Parameter analysis on remeshed SMAL~\cite{zuffi20173d}.}
\scalebox{1.0}{
\begin{tabular}{lr}
\hline
$\alpha$                & SMAL                \\ \hline
$\alpha$ =10     & 2.9       \\ 
$\alpha$ =1      & 2.6     \\
$\alpha$=0.1    & 2.6  \\
$\alpha$=0.07               & 2.6      \\ 
$\alpha$=0.01               & 2.7      \\ 
\hline
\end{tabular}}
\label{tab: parameter analysis}
\end{table}

\section{Inhibition Function Analysis}\label{sec:Inhibition_analysis}
The key to optimizing basis functions lies in learning a set of inhibition functions. To further investigate the role of the inhibition function, we visualize its learned quantitative results across different datasets in~\cref{fig:Inhibition_functions}. 
We draw the following key observations:
(1) Spectral Matching Hierarchy: Lower-frequency Laplacian eigenfunctions demonstrate greater significance than their higher-frequency counterparts in shape correspondence. Visualized inhibition functions consistently exhibit stronger attenuation on high-frequency components. (2) Near-isometric Adaptation: Under mild deformations (FAUST/SCAPE), high-frequency suppression remains moderate. The inhibition function shows gradual decay, with penalty factors below 0.5 appearing only beyond $\{\phi_k\}_{k\geq 100}$. (3) Non-isometric Response: Severe deformations (SMAL/DT4D-H) necessitate aggressive high-frequency suppression. The inhibition function drops sharply, applying sub-0.5 penalties to components beyond $\{\phi_k\}_{k\geq 10}$. (4) Topological Noise Characterization: TOPKIDS exhibits intermediate behavior—steeper decay than FAUST/SCAPE yet gentler than SMAL/DT4D-H, confirming its deformation level bridges near-isometric and strongly non-isometric scenarios. This conclusion aligns with the findings reported in DeepFAFM~\cite{luo2025deep}. 
Our approach adaptively suppresses the influence of noisy basis components according to task-specific requirements, thereby achieving superior performance compared to feature-learning-only approaches.
\begin{figure}[h!t]
	\centering
	\includegraphics[width=0.85\linewidth]{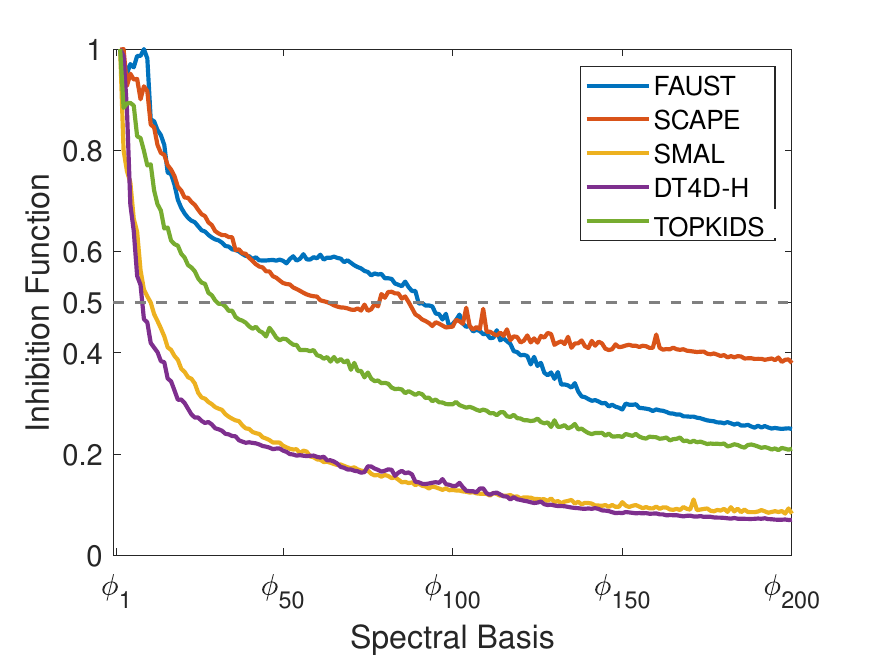}
	\caption{The quantitative results of learned inhibition functions across different datasets.}
	\label{fig:Inhibition_functions}
\end{figure}

\section{Qualitative Results}\label{sec:visualization}
In this section, we provide the qualitative results of our method on SHREC'19 (see~\cref{fig5: iso_gen}), SMAL, DT4D-H (see~\cref{fig6: non_iso}), and TOPKIDS (see~\cref{fig7: topo}) corresponding to the quantitative results reported in the main text.

\section{Future Work}\label{sec:limitation}
Several compelling avenues for future work emerge, such as: i) While this work adopts spectral convolution networks, the inhibition function can alternatively be parameterized by MLPs, recently proposed KANs~\cite{somvanshi2025survey}, attention (Transformer) mechanisms~\cite{chen2024survey,yuan2025survey}, or other machine learning models, all of which offer potential improvements to the basis representation;  
ii) Extending our method to non-orthogonal elastic basis~\cite{2023ElasticBasis}, which encodes extrinsic geometric information, can enhance robustness under non-isometric and topologically noisy settings compared to conventional Laplacian basis. We anticipate this extension will advance matching performance in challenging scenarios; iii) Integrating basis function optimization into partial shape matching frameworks~\cite{attaiki2021dpfm,ehm2024partial,xie2025echomatch} is expected to substantially improve their correspondence accuracy and robustness.

\begin{figure*}[h!t]
	\centering
	\includegraphics[width=0.7\linewidth]{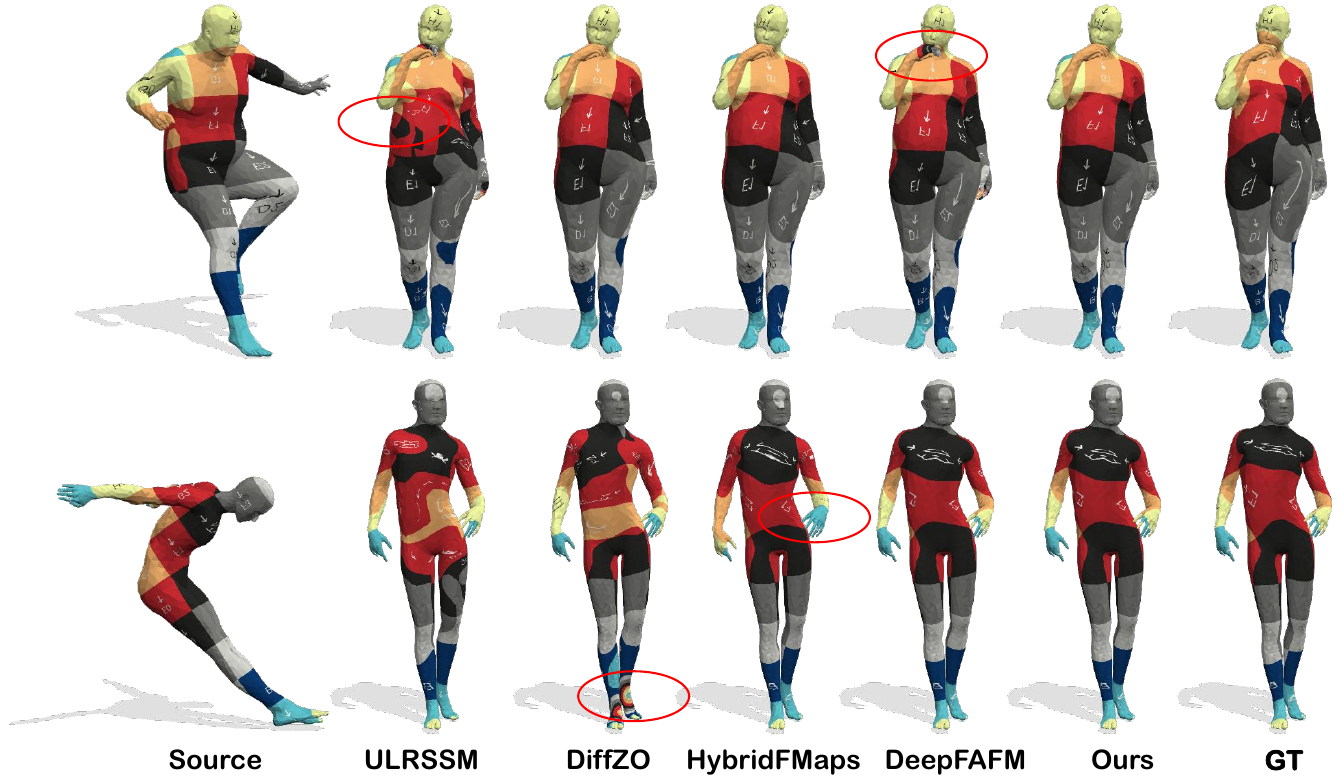}
	\caption{ Comparisons of cross-dataset generalization performance with other methods. Top row: training on FAUST~\cite{ren2018continuous} and testing on SHREC'19~\cite{melzi2019shrec}. Bottom row: Training on SCAPE~\cite{ren2018continuous} and testing on SHREC'19. Our method exhibits fewer errors and less color distortion compared to other approaches, highlighting its robust performance.    
    }
	\label{fig5: iso_gen}
\end{figure*}

\begin{figure*}[h!t]
	\centering
	\includegraphics[width=0.83\linewidth]{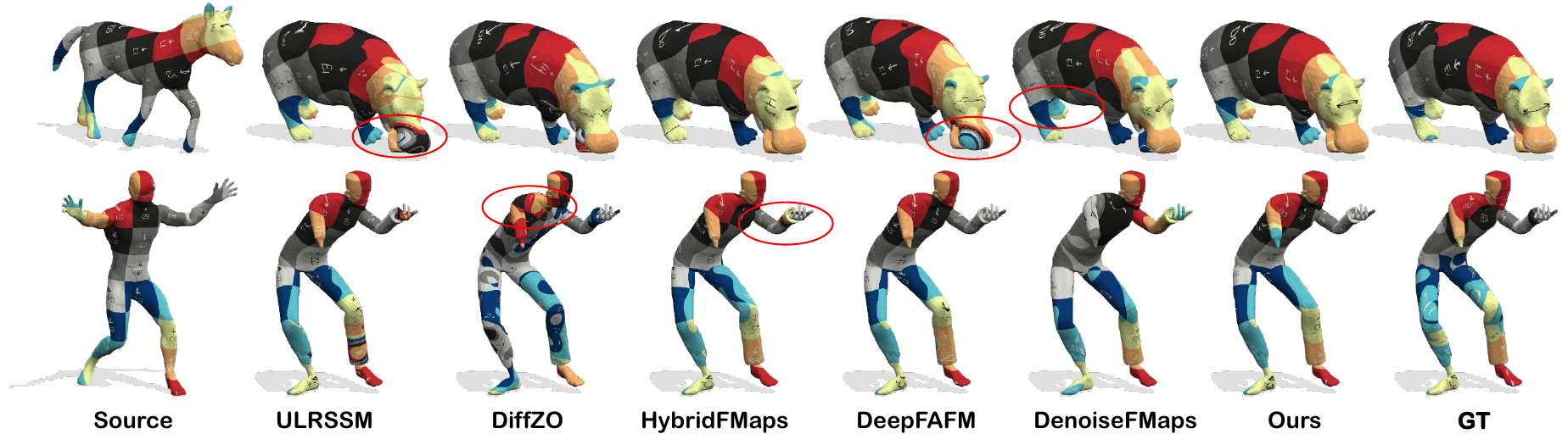}
	\caption{ Comparisons with other methods on the SMAL~\cite{zuffi2017} (top row) and DT4D-H~\cite{2022SmoothNonRigidShapeMatchingviaEffectiveDirichletEnergyOptimization} (bottom row) datasets. Our approach results in fewer errors and less texture distortion than other methods, demonstrating its superior performance for non-isometric shape matching.}
	\label{fig6: non_iso}
\end{figure*}

\begin{figure*}[h!t]
	\centering
	\includegraphics[width=0.7\linewidth]{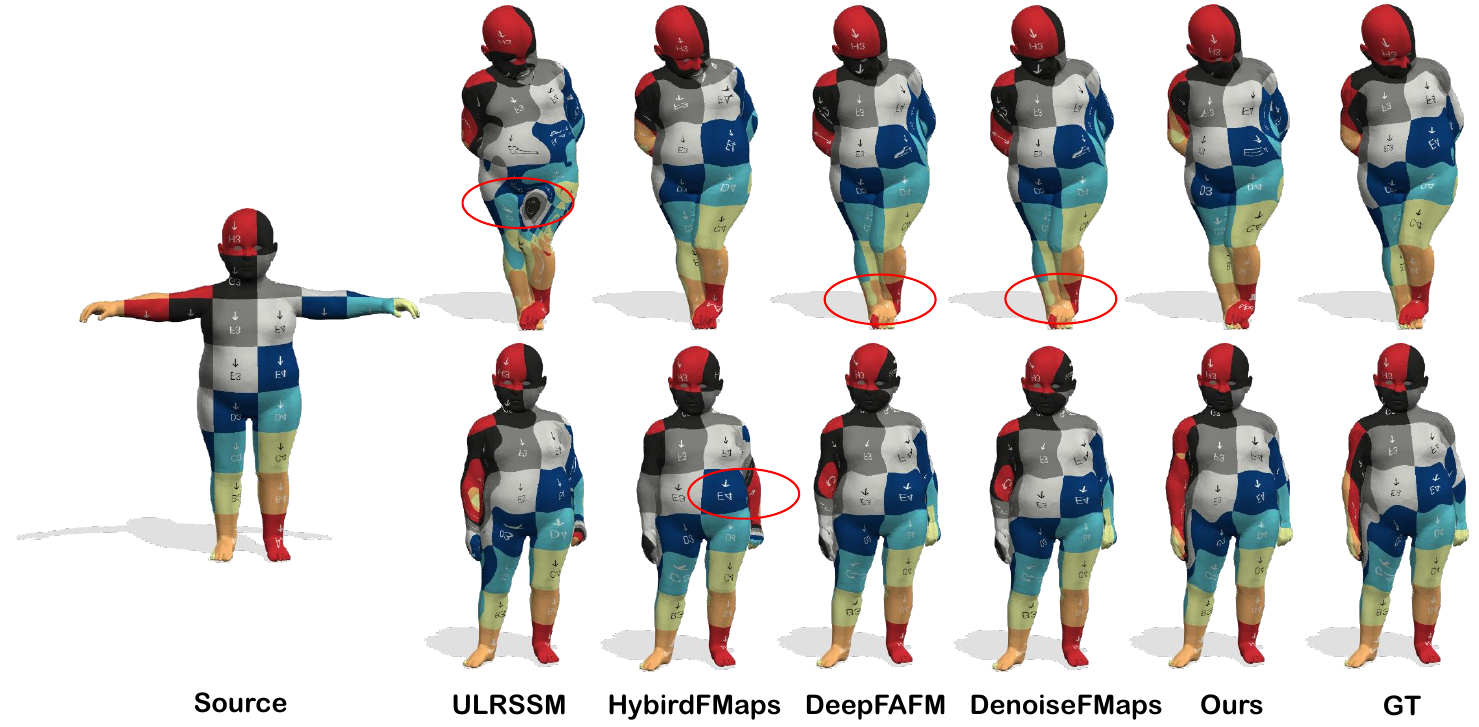}
	\caption{Comparisons with other methods on shape matching with topological noise~\cite{lahner2016shrec}. Our results, with smoother and more accurate texture distributions, illustrate that our approach is more robust to topological noise compared to existing methods
    }\label{fig7: topo}
\end{figure*}


\end{document}